\documentclass{article}

     \PassOptionsToPackage{numbers, compress}{natbib}
\usepackage[preprint]{neurips_2021}

\usepackage[utf8]{inputenc} % allow utf-8 input
\usepackage[T1]{fontenc}    % use 8-bit T1 fonts
\usepackage{hyperref}       % hyperlinks
\usepackage{url}            % simple URL typesetting
\usepackage{booktabs}       % professional-quality tables
\usepackage{amsfonts}       % blackboard math symbols
\usepackage{nicefrac}       % compact symbols for 1/2, etc.
\usepackage{microtype}      % microtypography
\usepackage{xcolor}         % colors

\usepackage{tikz}
\usepackage{subcaption}
\usepackage{microtype}
\usepackage{subdepth}
\usepackage{bbm,bm,courier}
\usepackage{setspace,relsize}
\usepackage{graphicx,placeins}
\usepackage{enumitem}
\usepackage{mathtools}
\usepackage{dsfont} % this is used for indicator function
\usepackage{titlesec}
\newcommand{\addperiod}[1]{#1.}
\titleformat{\paragraph}[runin]{\normalfont\bfseries\color{black}\setlength{\parindent}{0pt}}{\theparagraph}{0.25em}{\addperiod}

\usepackage{array,booktabs,tabularx,multirow,colortbl,hhline}
\newcommand{\cell}[2]{\setlength{\tabcolsep}{0pt}\begin{tabular}{#1}#2 \end{tabular}}

\newcommand{\sccell}[2]{\setlength{\tabcolsep}{0pt}\scshape\begin{tabular}{#1}#2\end{tabular}}

\usepackage{amsmath,amssymb}
\usepackage{eqnarray,mathtools}
\usepackage[oldsyntax]{stackengine}

\usepackage{amsthm}

\newtheorem{definition}{Definition}[section]

\newtheorem{theorem}{Theorem}
\newtheorem{lemma}{Lemma}
\newtheorem{proposition}{Proposition}

\newcommand{\textds}[1]{{\footnotesize\texttt{#1}}}

\setitemize{leftmargin=1em,parsep=2pt,label=\raisebox{0.25ex}{\tiny$\bullet$}}
\setlist[enumerate]{leftmargin=*, label= {\arabic*.}, itemsep=0.5em}
\newlist{thmlist}{enumerate}{1}
\setlist[thmlist]{leftmargin=*,label=\raisebox{0.25ex}{\tiny$\bullet$}, topsep=0.2em,itemsep=2pt}

\newcommand{\squishlist}{
\begin{list}{{{\small{$\bullet$}}}}
{\setlength{\itemsep}{3pt}      \setlength{\parsep}{1pt}
\setlength{\topsep}{1pt}       \setlength{\partopsep}{0pt}
\setlength{\leftmargin}{1em} \setlength{\labelwidth}{1em}
\setlength{\labelsep}{0.5em} } }
\newcommand{\squishend}{  \end{list}  }

\def\ourmethod{{\textsf{CA}}}

\usepackage[capitalise]{cleveref}
\usepackage{macros}

\usepackage{comment}

\title{Linear Classifiers that Encourage Constructive Adaptation}

\author{%
   Yatong Chen\\
  UC Santa Cruz\\
  Santa Cruz, CA 95064 \\
  \texttt{ychen592@ucsc.edu} \\
   \And
   Jialu Wang \\
  UC Santa Cruz\\
  Santa Cruz, CA 95064 \\
  \texttt{faldict@ucsc.edu} \\
   \And
   Yang Liu \\
  UC Santa Cruz\\
  Santa Cruz, CA 95064 \\
  \texttt{yangliu@ucsc.edu} \\
}

\begin{document}

\maketitle

\begin{abstract}
Machine learning systems are often used in settings where individuals adapt their features to obtain a desired outcome. 
In such settings, strategic behavior leads to a sharp loss in model performance in deployment.
In this work, we aim to address this problem by learning classifiers that encourage decision subjects to change their features in a way that leads to improvement in both predicted \emph{and} true outcome. 
We frame the dynamics of prediction and adaptation as a two-stage game, and characterize optimal strategies for the model designer and its decision subjects.
In benchmarks on simulated and real-world datasets, we find that classifiers trained using our method maintain the accuracy of existing approaches while inducing higher levels of improvement and less manipulation.
\end{abstract}

\section{Introduction}
\label{sec:intro}

Individuals subject to a classifier's predictions may act strategically to influence their predictions. Such behavior, often referred to as \emph{strategic manipulation} \cite{hardt2016strategic}, may lead to sharp deterioration in classification performance. However, not all strategic behavior is detrimental: in many applications, model designers stand to benefit from strategic adaptation if they deploy a classifier that incentivizes decision subjects to perform adaptations that improve their true outcome~\cite{haghtalab2020maximizing, shavit2020causal}. For example:
\begin{itemize}
\item \textbf{Lending}: In lending, a classifier predicts a loan applicant's ability to repay their loan. If the classifier is designed so as to incentivize the applicants to improve their income, it will also improve the likelihood of repayment.
\item \textbf{Content Moderation}: In online shopping, a recommender system suggests products to customers based on their relevance. Ideally, the algorithm should incentivize the product sellers to publish accurate product descriptions by aligning this with improved recommendation rankings.
\end{itemize}

In this work, we study the following mechanism design problem: a \emph{model designer} must train a classifier that will make predictions over \emph{decision subjects} who will alter their features to obtain a specific prediction. Our goal is to learn a classifier that is accurate and that incentivizes decision subjects to adapt their features in a way that improves both their predicted \emph{and} true outcomes.

Our main contributions are as follows:
\begin{enumerate}
% Conceptual
\item We introduce a new approach to handle strategic adaptation in machine learning, based on a new concept we call the \emph{constructive adaptation risk}, which trains classifiers that incentivize decision subjects to adapt their features in ways that improve true outcomes. We provide formal evidence that this risk captures both the strategic and constructive dimensions of decision subjects' behavior.

% Theory
\item We characterize the dynamics of strategic decision subjects and the model designer in a classification setting using a two-player sequential game.
Concretely, we provide closed-form optimal strategies for the decision subjects (\cref{thm:best-response-close-form}). The implications (\cref{sec:discussion}) reveal insights about the decision subjects' behaviors when the model designer uses non-causal features (features that don't affect the true outcome) as predictors.

% METHOD & EVALUATION
\item We formulate the problem of training such a desired classifier as a risk minimization problem. We evaluate our method on simulated and real-world datasets to demonstrate how it can be used to incentivize improvement or discourage adversarial manipulation. Our empirical results show that our method outperforms existing approaches, even when some feature types are misspecified.
\end{enumerate}

\subsection{Related work}
\label{Sec:RelatedWork}

Our paper builds on the strategic classification literature in machine learning~\citep{hardt2016strategic,cai2015optimum,ben2017best,chen2018strategyproof,dong2018strategic,Dekel2010Incentive,chen2020learning,tsirtsis2019optimal}. 
We study the interactions between a model designer and decision subjects using a a sequential two-player Stackelberg game~\citep[see e.g.,][for similar formulations]{hardt2016strategic, bruckner2011stackelberg,balcan2015commitment,dong2018strategic, tsirtsis2019optimal}.

We consider a setting where strategic adaptation can consist of manipulation as well as improvement. Our broader goal of designing a classifier that encourages improvement is characteristic of recent work in this area \citep[see e.g.,][]{kleinberg2020classifiers,haghtalab2020maximizing, shavit2020causal, rosenfeld2020predictions}. In general, it's hard to distinguish causal features (features that affect the true outcome) from non-causal features: Miller et al. \cite{miller2020strategic} show that designing an improvement-incentivizing model requires solving a non-trivial causal inference problem. 

This paper also broadly relates to work on recourse \citep{ustun2019actionable, venkatasubramanian2020philosophical,karimi2020survey,gupta2019equalizing,karimi2020algorithmic, vonkugelgen2020fairness} in that we aim to fit models that provide \emph{constructive recourse}, i.e. actions that allow decision subjects to improve both their predicted \emph{and} true outcomes.
Our approach may be useful for mitigating the disparate effects of strategic adaptation~\cite{hu2019disparate,Milli2019social,liu2020disparate} that stem from differences in the cost of manipulation (see \cref{prop:4}). Lastly, our results may be helpful for developing robust classifiers in dynamic environments, where both decision subjects' features and the deployed models may vary across time periods~\citep{kilbertus2020fair,shavit2020causal,liu2017machine}.

Also relevant is the recent work on performative prediction \cite{perdomo2020performative,miller2021outside,izzo2021performative,mendler2020stochastic}, in which the choice of model itself affects the distribution over instances. However, this literature differs from ours in that we focus on inducing constructive adaptations from decision subjects, rather than finding a policy that incurs the minimum deployment error. In addition, our formulation arguably requires less knowledge, is more intuitive and deployable, and requires fewer assumptions on the loss function.
% \yl{explain PP requires too many assumption to solve the exact adaptation problem. Our formulation arguably requires less knowledge, more intuitive and deployable, and we provide justification in Section \ref{sec:method} that ours is a good approximation to the adapted risk introduced in \cite{perdomo2020performative}
%  }
\section{Problem statement}
\label{sec:problem-statement}

In this section, we describe our approach to training a classifier that encourages constructive recourse in settings with strategic adaptation. 

\subsection{Preliminaries}

We consider a standard classification task of training a classifier $h: \mathbb{R}^d \to \{-1,+1\}$ from a dataset of $n$ examples $(x_i,y_i )_{i=1}^n$, where example $i$ consists of a vector of $d$ features $x_i \in \RealNumber^d$  and a binary label $y_i \in \{-1,+1\}.$ Example $i$ corresponds to a person who wishes to receive a positive prediction $h(x_i) = +1$, and who will alter their features to obtain such a prediction once the model is deployed.

We formalize these dynamics as a sequential game between the following two players:
\begin{enumerate}
    \item A model designer, who trains a classifier $h: \X \rightarrow \{-1, +1\}$ from a hypothesis class $\Hset$.
    
    \item Decision Subjects, who adapt their features from $x$ to $x'$ so as to be assigned $h(x') = +1$ if possible. We assume that decision subjects incur a cost for altering their features, which we represent using a \emph{cost function} $c: \X \times \X \rightarrow \RealNumber^+$.
\end{enumerate}
We assume that each player has complete information: decision subjects know the model designer's classifier, and the model designer knows the decision subjects' cost function.
Decision subjects alter their features based on their current features $x$, the cost function $c$, and the classifier $h$, so that their altered features can be written $\SB = \Delta(x; h, c)$ where $\Delta(\cdot)$ is the \emph{best response function}.

We allow adaptations that alter the true outcome $y$. To describe these effects, we refer to the \emph{true label function} $y: \mathcal X \rightarrow \{-1, +1\}$, such that $y_i = y(x_i)$. In practice, $y(\cdot)$ is unknown; however, our approach will involve assumptions about how altering a feature affects the true outcome.

\subsection{Background}

In a standard prediction setting, a model designer trains a classifier that minimizes the \emph{empirical risk}:
\begin{align*}
    h_{\textsf{ERM}}^* \in \argmin_{h \in \Hset} R_{\textsf{ERM}}(h) 
\end{align*}
where $R_{\textsf{ERM}}(h) = \Expectation_{x\sim \Dataset}[\mathbbm{1}(h(x) \neq y)]$. This classifier performs poorly in a setting with strategic adaptation, since the model is deployed on a population with a different distribution over $\mathcal{X}$ (as decision subjects alter their features) and $y$ (as changes in features may alter true outcomes).

Existing approaches in strategic classification tackle these issues by training a classifier that is robust to \emph{all} adaptation. This approach treats all adaptation as undesirable, and seeks to maximize accuracy by discouraging it entirely. Formally, they train a classifier that minimizes the \emph{strategic risk}:
\begin{align*}
    h_{\textsf{SC}}^* \in \argmin_{h \in \Hset} R_{\textsf{SC}}(h)
\end{align*}
where $R_{\textsf{SC}}(h) = \Expectation_{x\sim \Dataset}[\mathbbm{1}(h(x_*) \neq y)]$, and  $x_* = \Delta(x, h ; c)$ denotes the features of a decision subject after adaptation. However, this classifier still has suboptimal accuracy because $y$ changes as a result of the adaptation in $x$. Further, this design choice misses the opportunity to encourage a profile $x$ to truly improve to change their $y$.

\subsection{\textsf{CA} risk: minimizing error while encouraging constructive adaptation}
\label{sec:proposed approach}
In many applications, model designers are better off when decision subjects adapt their features in a way that yields a specific true outcome, such as $y = +1$.  Consider a typical lending application where a model is used to predict whether a customer will repay a loan. In this case, a model designer benefits from $y = +1$, as this means that a borrower will repay their loan.

To help explain our proposed approach, we assume that we can write $x = [\XI ~|~ \XM ~|~ \XIM ]$ where $\XI$, $\XM$ and $\XIM$ denote the following categories of features: 
\begin{itemize}
\item  \emph{Immutable} features (\XIM), which cannot be altered (e.g. race, age).

\item \emph{Improvable} features (\XI), which can be altered in a way that will either increase or decrease the true outcome (e.g. education level, which can be increased to improve the probability of repayment).

\item \emph{Manipulable} features (\XM), which can be altered without changing the true outcome (e.g. social media presence, which can be used as a proxy for influence). Notice that it is the \emph{change} in these features that is undesirable; the features themselves may still be useful for prediction.

There may also be features that can be altered but whose effect is \emph{unknown}. In this work, we treat them as manipulable features.
\end{itemize}

We also use $\XA = [\XI ~|~ \XM]$ to denote the \emph{actionable} features, and $\DA$ to denote its dimension. 

Note that the question of how to decide which features are of which type is beyond the scope of the present work; however, this is the topic of intense study in the causal inference literature \cite{miller2020strategic}. Analogously, we define the following variants of the best response function $\Delta$:
\begin{itemize}
    \item $\SBI = \BI(x, h ; c)$: the \emph{improving best response}, which involves an adaptation that only alters improvable features.
    \item $\SBM = \BM(x, h ; c)$: the \emph{manipulating best response},  which involves an adaptation that only alters manipulable features.
\end{itemize}
Note that in reality, a decision subject can still alter both types of features, which means that they will perform $\Delta(x, h; c)$, unless the model designer explicitly forbids changing certain features. However, it still worth distinguishing different types of best responses when the model designer designs the classifier: we can think of the improving best response $\BI$ as the best possible adaptation which only consists of honest improvement, while the manipulating best response $\BM$ is the worst possible adaptation that consists of pure manipulation. The model designer would like to design a classifier such that for the decision subjects, $\Delta(x, h; c)$ appears to be close to $\BI(x, h ; c)$. 

We train a classifier that balances between robustness to manipulation and incentivizing improvement:
\begin{align}
\label{eqn:general objective function}
   h_{\ourmethod}^* = \argmin_{h \in \Hset} [R_\M(h) + \lambda \cdot R_\I(h)],
\end{align}
The first term, $R_\M(h) = \Expectation_{x \sim \Dataset}[\Indicator(h(\SBM) \neq y)]$, is the \emph{manipulation risk}, which penalizes pure manipulation.
The second term, $R_\I(h) = \Expectation_{x\sim \Dataset}[\Indicator(h(\SBI) = +1)]$, is the \emph{improvement risk}, which rewards decision subjects for playing their improving best response.
The parameter $\lambda > 0$ trades off between these competing objectives. Setting $\lambda \rightarrow 0$ results in an objective that simply discourages manipulation, whereas increasing $\lambda \rightarrow \infty$ yields a trivial classifier that always predicts $+1$.

The two terms in the objective function can also be viewed as proxies for other familiar notions. In Section \ref{subsec:model-designers-program}, we show that under reasonable conditions, the following hold:
\begin{itemize}
    \item The first term, $R_\M(h)$, is an upper bound on $R_{\textsf{SC}}(h)$. Thus minimizing the manipulation risk also minimizes the traditional strategic risk.
    \item A decrease in the second term, $R_\I(h)$ reflects an increase in $\Probability(y(\SBI) = +1)$. Thus improvement in the prediction outcome aligns with improvement in the true qualification.
\end{itemize}
\section{Decision subjects' best response}
\label{sec:agent}

In this section, we characterize the decision subjects' best response function. Proofs for all results are included in \cref{sec:sec3-proof}.

\subsection{{Setup}}
\label{subsec:assumption}

We restrict our analysis to the setting in which a model designer trains a \emph{linear classifier} $h(x) = \Sign(w^{\T} x ),$ where $w = [w_0, w_1, \ldots, w_d] \in \RealNumber^{d+1}$ denotes a vector of $d+1$ weights.

We capture the cost of altering $x$ to $x'$ through the \emph{Mahalanobis} norm of the changes:\footnote{Since immutable features $\XIM$ cannot be altered, the cost function involves only the actionable features $\XA$.}
\begin{align*}
c(x , x') = \sqrt{(\XA - \XA')^{\T} S^{-1}(\XA - \XA')}
\end{align*}
Here, $S^{-1} \in \RealNumber^{\DA} \times \RealNumber^{\DA}$ is a symmetric \emph{cost covariance matrix} in which $S^{-1}_{j,k}$ represents the cost of altering features $j$ and $k$ simultaneously. To ensure that $c(\cdot)$ is a valid norm, we require $S^{-1}$ to be \emph{positive definite}, meaning $\XA^{\T}S^{-1}\XA > 0$ for all $\XA \neq \mathbf{0} \in \RealNumber^{\DA}$. Additionally, to prevent correlations between improvable and manipulable features, we assume $S^{-1}$ is a diagonal block matrix of the form
\begin{align} 
\label{eqn:S}
    S^{-1} = 
        \begin{bmatrix}
            \SI^{-1} & 0\\
            0 & \SM^{-1}
        \end{bmatrix}, \quad \text{which also implies} \quad
            S = \begin{bmatrix}
            \SI & 0\\
            0 & \SM
        \end{bmatrix}~
\end{align}
%
%The non-diagonal entries let us specify correlation between features. For example, improving one's income potentially also helps improve payback capability. 

Otherwise, we allow the cost matrix to contain non-zero elements on non-diagonal entries. This means that our results hold even when there are interaction effects when altering multiple features. This generalizes prior work on strategic classification in which the cost is based on the $\ell_2$ norm of the changes, which is tantamount to setting $S^{-1} = I$, and therefore assumes the change in each feature contributes independently to the overall cost \citep[see e.g.,][]{hardt2016strategic,haghtalab2020maximizing}.
%when using $L_2$ norm, for example, the change in each feature will incurs a cost that independently contribute to the overall cost. This is not true in general. For example, an 
%increase in an individual's education level may lead to an increase in their chances at employment.
%That is, the cost of being employed and the cost of increasing an individual's education level are correlated in a certain way. 
%[1 sentence to describe how this allows you to model things better]
%Our cost function will be able to help us model the correlations between changes in features explicitly by using the cost-covariate matrix $S$, and we can achieve this by imposing an non-zero terms in the corresponding entries in the cost matrix $S$.
%Using the Mahalanobis norm enables us to analyze \emph{fairness} in the strategic recourse setting. For example, when two social groups have the same inherent qualifications --- i.e., they share the same costs of improvement $S_{\I}$ --- but experience different manipulation costs $S_{\M}$, how does this affect the classification results as well as the two groups' ability to achieve recourse? We highlight the effect of differences in cost functions, as well as how our model can mitigate these disparities, empirically in \cref{sec:empirical}.

\subsection{Decision subject's best response model}\label{sec:brm}

Given the assumptions of \cref{subsec:assumption}, we can define and analyze the decision subjects' best response. %These expressions will let us encode the decision subjects' behavior into the model designer's objective function in Section \ref{sec:method}.
We start by defining the decision subject's payoff function. Given a classifier $h$, a decision subject who alters their features from $x$ to $x'$ derives total utility
\begin{align*}
    U (x, x') = h(x') - c(x, x')
\end{align*}
%
%If a decision subject does not alter $x$, this will result in a utility of $h(x)$ as $c(x,x) \equiv 0$. 
Naturally, a decision subject tries to maximize their utility; that is, they play their \emph{best response}:
\begin{definition}[$\F$-Best Response Function]
Let $\F \in \{\I, \M, \A\}$, and let $\mathcal X_{\F}^*(x)$ denote the set of vectors that differ from $x$ only in features of type \F. Let $\BF: \mathcal X \rightarrow \mathcal X$ denote the \F-\emph{best response} of a decision subject with features $x$ to $h$, defined as:
\begin{align*}
    \BF (x) = 
    \argmax\limits_{x'\in \mathcal X^*_{\F}(x)} U(x, x')
\end{align*}
\end{definition}
Setting $\F=\I$ gives the \emph{improving best response} $\BI(x)$, in which the adaptation changes only the improvable features; setting $\F=\M$ yields the \emph{manipulating best response} $\BM(x)$, in which only manipulable features are changed. Setting $\F=\A$, we get the standard \emph{unconstrained best response} $\BU(x)$ in which any actionable features can be changed. As we mentioned earlier, we will also use $\SBF:= \BF(x)$ as shorthand for the \F-best response, and we denote $\Delta(x) := \BU(x)$.

Intuitively, the cost of manipulation should be smaller than the cost of actual improvement. For example, improving one's coding skills should take more effort, and thus be more costly, than simply memorizing answers to coding problems. As a result, one would expect the gaming best response $\BM(x)$ and the unconstrained best response $\Delta(x)$ to flip a negative decision more easily than the improving best response $\BI(x)$. In \cref{sec:discussion}, we formalize this notion (\cref{prop:2}).

We prove the following theorem characterizing the decision subject's different best responses:

\begin{theorem}[$\F$-Best Response in Closed-Form]
\label{thm:best-response-close-form}
Given a linear threshold function $h(x) = \Sign(w^\T x)$ and a decision subject with features $x$ such that $h(x) = -1$, reorder the features so that $x = [\XAF ~|~ \XF ~|~ \XIM]$, and let $\CF = \WF^\T \SF\WF$. Then $x$ has $\F$-best response
\begin{align}\label{eq:best_response_delta_x}
    \BF(x) =
        \begin{cases}
             \left[\XF - \frac{w^{\T}x}{\CF}\SF\WF\right] ~|~  \XAF ~|~ \XIM, & \text{if $\frac{|w^{\T}x|}{\sqrt{\CF}}\leq 2$}\\
           x, & \text{otherwise}
        \end{cases}
\end{align} 
with corresponding cost
\begin{align*}
    c(x, \BF(x)) = 
        \begin{cases}
            \frac{|w^{\T}x|}{\sqrt{\CF}}, & \text{if $\frac{|w^{\T}x|}{\sqrt{\CF}}\leq 2$}\\
            0 & \text{otherwise}
        \end{cases}
\end{align*}
\end{theorem}
\textit{Example:} When $\F = \M$, $\XF = \XM$ and $\XAF =[\XI ~|~ \XIM]$~.
After reordering features, we get the following closed-form expression for the manipulating best response:
\begin{align*}
    \BM(x) = 
    \begin{cases}
         \left[ \XI ~|~ \XM - \tfrac{w^{\T} x }{\CM} \SM \WM  ~|~  \XIM \right] \  & \text{if}\ \ \tfrac{|w^{\T} x |}{\sqrt{\CM}} \leq 2 \\
        x, \ & \text{otherwise}
    \end{cases}
\end{align*}
with corresponding cost
\begin{align*}
    c(x, \BM(x)) = 
         \begin{cases}
            \frac{|w^{\T}x|}{\sqrt{\CM}}, & \text{if}\ \ \tfrac{|w^{\T} x|}{\sqrt{\CM}} \leq 2 \\
            0 & \text{otherwise}
        \end{cases}
\end{align*}
%
% Similarly when $\F = \I$, $\XF = \XI$ and $\XAF =[\XM ~|~ \XIM]$, and we get the improving best response $\BI(x) = \left[\XI - \tfrac{w^{\T} x }{\CI} \SI \WI ~|~ \XM ~|~ \XIM \right]$ if the cost is less than $2$.

%%%%%%%%%%%%%%%%%%%%%%%%%%%%%%%%%%%%%%%%%%%%%%%%%%%%%%%%%%%%%%%%
% Page 6
%%%%%%%%%%%%%%%%%%%%%%%%%%%%%%%%%%%%%%%%%%%%%%%%%%%%%%%%%%%%%%%%

\subsection{Discussion}
\label{sec:discussion}
%We now discuss the implications of prior results. We include the proofs in \cref{sec:discussion-proof}.

In Proposition \ref{prop:1}, we demonstrate a basic limitation for the model designer: if the classifier uses any manipulable features as predictors, then decision subjects will find a way to exploit them. Hence the only way to avoid any possibility of manipulation is to train a classifier without such features.
\begin{proposition}[Preventing Manipulation is Hard]
\label{prop:1}
Suppose there exists a manipulated feature $x^{(m)}$ whose weight in the classifier $w_\A^{(m)}$ is nonzero. Then for almost every $x \in \mathcal X$, $\Delta^{(m)}(x)\neq x^{(m)}$. 
\end{proposition}

Next, we show that the unconstrained best response $\Delta(x)$ dominates the improving best response $\BI(x)$, thus highlighting the difficulty of inducing decision subjects to change only their improvable features when they are also allowed to change manipulable features.
%\yl{add transitioning sentences between propositions - refer to what I had before proposition 3.}
\begin{proposition}[Unconstrained Best Response Dominates Improving Best Response]
\label{prop:2}
Suppose there exists a manipulable feature $x^{(m)}$ whose weight in the classifier $w_A^{(m)}$ is nonzero. Then, if a decision subject can flip her decision by playing the improving best response, she can also do so by playing the unconstrained best response. The converse is not true: there exist decision subjects who can flip their predictions through their unconstrained best response but not their improving best response.
\end{proposition} 

Next, we show how correlations between features affect the cost of adaptation. This can be demonstrated by looking at any cost matrix and adding a small nonzero quantity $\tau$ to some $i,j$-th and $j,i$-th entries. Such a perturbation can reduce every decision subject's best-response cost:
\begin{proposition}[Correlations between Features May Reduce Cost]
\label{prop:3}
For any cost matrix $S^{-1}$ and any nontrivial classifier $h$, there exist indices $k,\ell \in [\DA]$ and $\tau \in \RealNumber$ such that every feature vector $x$ has lower best-response cost under the cost matrix $\tilde{S}^{-1}$ given by
\[
\tilde{S}^{-1}_{ij} = \tilde{S}^{-1}_{ji} =
\begin{cases}
    S^{-1}_{ij} + \tau, & \text{ if } i=k, j=\ell \\
    S^{-1}_{ij},        & \text{ otherwise}
\end{cases}
\]
than under $S^{-1}$; that is, $c_{\tilde{S}^{-1}}(x,\Delta(x)) < c_{S^{-1}}(x,\Delta(x))$ for all $x$.
\end{proposition}
In many applications, decision subjects may incur different costs for modifying their features, resulting in disparities in prediction outcomes~\citep[see][for a discussion]{hu2019disparate}. To formalize this phenomenon, suppose $\Phi$ and $\Psi$ are two groups whose costs of changing improvable features are identical, but members of $\Phi$ incur higher costs for changing manipulable features. Let $\phi \in \Phi$ and $\psi \in \Psi$ be two people from these groups who share the same profile, i.e. $x_\phi = x_\psi$. We show the following:
\begin{proposition}[Cost Disparities between Subgroups]
\label{prop:4}
Suppose there exists a manipulated feature $x^{(m)}$ whose corresponding weight in the classifier $w_\A^{(m)}$ is nonzero. Then if decision subjects are allowed to modify any features, $\phi$ must pay a higher cost than $\psi$ to flip their classification decision.
\end{proposition} 
\cref{prop:4} highlights the importance for a model designer to account for these differences when serving a population with heterogeneous subgroups. 
\section{Constructive adaptation risk minimization}
\label{sec:method}
% Write ERM with Best Response + Hinge Loss
In this section we analyze the training objective for the model designer, formulating it as an empirical risk minimization (ERM) problem. Any omitted details can be found in \cref{sec:sec4-proof}.

\subsection{The model designer's program}
\label{subsec:model-designers-program}

The model designer's goal is to publish a classifier $h$ that maximizes the classification accuracy while incentivizing individuals to change their improvable features. By Theorem \ref{thm:best-response-close-form}, we have 
\begin{align}
    \SBM & = 
    \begin{cases}
    \label{eqn:best-response-M}
        \left[ \XI ~|~  \XM - \frac{w^{\T} x}{\CM} \SM \WM ~|~ \XIM \right] & \text{if} \frac{|w^\T x|}{\sqrt{\CM}} \leq 2 \\
        x, & \text{otherwise}
    \end{cases}\\
    \SBI & = 
    \begin{cases}
     \label{eqn:best-response-I}
        \left[\XI - \frac{w^{\T} x}{\CI} \SI \WI ~|~  \XM ~|~  \XIM \right], &\text{if}~ \frac{|w^{\T} x|}{\sqrt{\CI}} \leq 2 \\
        x, & \text{otherwise}
    \end{cases}
\end{align}

Recall from \cref{sec:proposed approach} that the model designer's optimization program is as follows:
\begin{align}
    \label{eq:objective}
        \min_{h\in\Hset} & \quad \Expectation_{x\sim \Dataset}\left[\Indicator(h(\SBM) \neq y)\right] + \lambda \Expectation_{x\sim \Dataset}\left[\Indicator(h(\SBI) \neq +1)\right]\nonumber \\ 
     \SuchThat  &\quad \SBM~\text{in \cref{eqn:best-response-M}},~ \SBI~\text{in \cref{eqn:best-response-I}}
\end{align}

\paragraph{Interpreting the objective}
The two terms in the objective function can be viewed as proxies for two other familiar objectives. The first term, $\Expectation_{x\sim \Dataset}\left[\Indicator(h(\SBM) \neq y)\right]$, directly penalizes pure manipulation. But as the following proposition suggests, minimizing this term also minimizes the traditional strategic risk when the true qualification does not change:

\begin{proposition}
\label{prop:manipulation-risk}
Assume that the manipulating best response is more likely to result in a positive prediction than the unconstrained best response, given that the true labels do not change. Then
\begin{align*}
    \Expectation_{x\sim\Dataset}\left[\Indicator[h(\SB)\neq y] ~|~ \Delta(y) = y\right] \leq \Expectation_{x\sim \Dataset}\left[\Indicator(h(\SBM) \neq y)\right] .
\end{align*}
\end{proposition}

The second term, $\Expectation_{x\sim \Dataset}\left[\Indicator(h(\SBI) \neq +1)\right]$, explicitly rewards decision subjects for playing their improving best response (closely related to the notion of \emph{recourse}). Of course, without positing a causal graph, we cannot know when $\Delta_\I(Y) = +1$; however, in the setting of \emph{covariate shift}, in which the distribution of $X$ may change but not the conditional label distribution $\Pr(Y|X)$, we can show that an increase in $\Probability(h(X) = +1)$ reflects an increase in $\Probability(Y = +1)$. This gives formal evidence that our prediction outcome aligns with improvement in the true qualification.

\begin{proposition}
\label{prop:improvement-risk}
Let $\Dataset^*$ be the new distribution after decision subject's best response. Denote $\omega_h(x) = \frac{\Probability_{\Dataset^*}(X = x)}{\Probability_{\Dataset}(X = x)}$ denote the amount of adaptation induced at feature vector $x$. Suppose $y(X)$ and $h(X)$ are both positively correlated with $\omega_h(X)$, and that $\Probability(Y|X)$ is the same before and after adaptation (the covariate shift assumption). Then the following are equivalent:
\begin{align*}
    \Pr[h(\SBI)=+1] > \Pr[h(x)=+1]
    \Longleftrightarrow
    \Pr[y(\SBI)=+1] > \Pr[y(x)=+1].
\end{align*}
\end{proposition}

Proofs of Propositions \ref{prop:manipulation-risk} and \ref{prop:improvement-risk} can be found in Appendix \ref{subsec:proof-of-prop-manipulation-risk} and \ref{subsec:proof-of-prop-improvement-risk}.

\subsection{Making the program tractable}

By substituting in the closed-form best responses for the decision subjects and making further mathematical steps (see Appendix \ref{sec:model-designer-objective-derivations} for details), we can turn the model designer's \emph{constrained} optimization problem in \eqref{eq:objective} into the following \emph{unconstrained} problem:
\begin{align}
\min_{w \in \RealNumber^{d+1}} \Expectation_{x\sim \Dataset} \Bigg[-
     \left(2\cdot \Indicator\left[w^{\T} x \geq -2\sqrt{\CM}\right] - 1 \right)\cdot y  - 2\lambda \cdot \Indicator\left[w^{\T}x \geq -2\sqrt{\CI}\right]\Bigg]
     \label{equ: key objective function}
\end{align}
The optimization problem in \eqref{equ: key objective function} is intractable since both the objective and the constraints are non-convex. To overcome this difficulty, we train our classifier by replacing the 0-1 loss function with a convex surrogate loss $\sigma(x) = \log\left(\frac{1}{1+e^{-x}}\right)$. This results in the following ERM problem:
\begin{equation}
\label{eqn: general ERM}
 \tilde{R}^{\star}_{\mathcal D}(h,\lambda) = \min_{w \in \RealNumber^{d+1}}
\frac{1}{n} \sum_{i=1}^n \Bigg[
   - \sigma\left(y_i \cdot (w^\T \cdot x_i + 2\sqrt{\CM})\right) 
    - \lambda \cdot \sigma(w^\T \cdot x_i + 2\sqrt{\CI})\Bigg] 
\end{equation}

\paragraph{Directionally Actionable Features} An additional challenge arises when some features can be changed in either a positive or negative direction, but not both (e.g. \texttt{has\_phd} can only go from \emph{false} to \emph{true}). In Appendix \ref{sec:partially-actionable-feature}, we show how to augment the above objective to enforce such constraints.

\section{Experiments}
\label{sec:empirical}

In this section, we present empirical results to benchmark our method on synthetic and real-world datasets. We test the effectiveness of our approach in terms of its ability to incentivize improvement (or disincentivize manipulation) and compare its performance with other standard approaches.  
Our submission includes all datasets, scripts, and source code used to reproduce the results in this section.

\begin{minipage}[htb]{\textwidth}
\begin{minipage}[t]{0.32\textwidth}
\begin{tikzpicture}[node distance={15mm}, thick, main/.style = {draw, circle}]
\node[main](z1){$Z_1$};
\node[main](x1)[right of=z1]{$X_1$};
\node[main](y)[below right of=x1]{$Y$};
\node[main](x2)[below left of=y]{$X_2$};
\node[main](z2)[left of=x2]{$Z_2$};
%\node[main](m)[right of=y]{$M$};
\node[main](m1)[above right of=y]{$M_1$};
\node[main](m2)[below right of=y]{$M_2$};
\draw[->] (z1) -- (x1);
\draw[->] (z2) -- (x2);
\draw[->] (x1) -- (y);
\draw[->] (x2) -- (y);
\draw[->] (x2) -- (m2);
\draw[->] (y) -- (m1);
\draw[->] (y) -- (m2);
\end{tikzpicture}

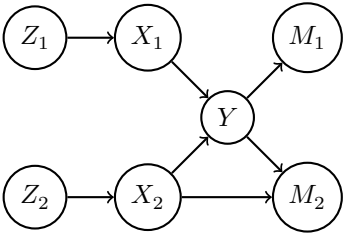
\captionof{figure}{\small A causal DAG for the \textds{toy} dataset. $Z_1$ and $Z_2$ are causal features that determine the true qualification $Y$, $X_1 = Z_1$, and $X_2$ is a noisy proxy for $Z_1$. We can directly observe $X_1$ and $X_2$ but not $Z_1$ or $Z_2$. $M_1$ and $M_2$ are non-causal features that correlate with $Y$ but do not influence it.}
\label{fig:causal-DAG-Z}
\end{minipage}
\hfill
\begin{minipage}[t]{0.64\textwidth}
\definecolor{best}{HTML}{BAFFCD}
\definecolor{issue}{HTML}{FFC8BA}
\definecolor{bad}{HTML}{FFC8BA}
\newcommand{\good}[1]{\cellcolor{best}#1} 
\newcommand{\bad}[1]{\cellcolor{issue}#1} 
\newcommand{\violation}[1]{\cellcolor{bad}#1} 
\newcommand{\basic}[0]{\textsf{ST}}
\newcommand{\dropfeatures}[0]{\textsf{DF}}
\newcommand{\manipulationproof}[0]{\textsf{MP}}
\newcommand{\lighttouch}[0]{\textsf{LT}}
\newcommand{\metrics}[0]{{\cell{l}{\textit{test error} \\ \textit{deployment error} \\ \textit{improvement rate}}}}

\normalsize 
\centering
\vspace{-1.2in}
\captionof{table}{\small Performance metrics for different specifications (\textbf{Spec.}) in which features may be misspecified. \basic~ denotes Static, \dropfeatures~ denotes DropFeatures, \manipulationproof~ denotes ManipulationProof, and \ourmethod~ denotes our method. For each method, we report \textit{test error}, \textit{deployment error}, and \textit{improvement rate}. In Full, the model designer has full knowledge of the causal DAG. In Mis. I, $M_1$ is mistaken for an improvable feature. In Mis. II, the improvable feature $X_1$ is miscategorized as manipulable.}
\label{table:toy_example_results}
\resizebox{0.8\linewidth}{!}{
\begin{tabular}{l l c c c c}
    % \normalsize 
    \toprule
         & & \multicolumn{4}{c}{\textsc{Methods}} \\
        \cmidrule(lr){3-6}
        \textbf{Spec.} & \textbf{Metrics} & \basic & \dropfeatures & \manipulationproof & \ourmethod \\
    \midrule
        \cell{c}{Full} & \metrics & \cell{c}{$10.29$ \\ \bad{$35.79$} \\ \bad{$11.54$}} & \cell{c}{\bad{$28.0$} \\ $35.15$ \\ $13.13$} & \cell{c}{$11.91$ \\ $24.1$ \\ $14.63$} & \cell{c}{$10.19$ \\ \good{$20.61$} \\ \good{$23.49$}} \\
    \midrule
        \cell{c}{Mis. I} & \metrics & \cell{c}{11.39 \\ \bad{37.37} \\ 37.23} & \cell{c}{10.52 \\ \good{10.53} \\ \good{39.74}} & \cell{c}{ 11.26 \\ 19.79 \\ \bad{0.62}} & \cell{c}{11.04 \\ 25.30 \\ 23.04} \\
    \midrule
        \cell{c}{Mis. II} & \metrics & \cell{c}{\good{10.58} \\ \good{12.37} \\ \bad{1.12}} & \cell{c}{\bad{35.77} \\ \bad{41.51} \\ 5.74} & \cell{c}{29.52 \\ 27.68 \\ 3.36 } & \cell{c}{\good{10.80} \\ 23.58 \\ \good{19.82}} \\
    \bottomrule
\end{tabular}
}
\end{minipage}
\end{minipage}

\subsection{Setup}
\paragraph{Datasets} 

We consider five datasets: \textds{toy}, a synthetic dataset based on the causal DAG in \cref{fig:causal-DAG-Z}; \textds{credit}, a dataset for predicting whether an individual will default on an upcoming credit payment~\cite{yeh2009comparisons}; \textds{adult}, a census-based dataset for predicting adult annual incomes; \textds{german}, a dataset to assess credit risk in loans; and \textds{spambase}, a dataset for email spam detection. The last three are from the UCI ML Repository~\cite{dua2017uci}. We provide a detailed description of each dataset along with a partitioning of features in \cref{tab:dataset-info} in the Appendix. We assume the cost of manipulation is lower than that of improvement, and that there are no correlations within the two types of adaptation; specifically, we use cost matrices $\SI^{-1} = I$ and $\SM^{-1} = 0.2 I$.
% \begin{figure}[htb]
%     \centering
%     \includegraphics[width=0.4\linewidth]{Figures/causal-DAG-Z.pdf}
%     \caption{A causal DAG for the \textds{toy} dataset. We assume that $Z_1$ and $Z_2$ are causal features that determine the true qualification $Y$. $X_1 = Z_1$ and $X_2$ is a noisy proxy for $Z_1$. We can directly observe $X_1$ and $X_2$ but not $Z_1,Z_2$. We also have two non-causal features $M_1$ and $M_2$, which do not affect $Y$ but are correlated with it. \textcolor{red}{(Latex this diagram)}}
%     \label{fig:causal-DAG-Z}
% \end{figure}
%
In our context, all we require is the knowledge that $X_1,X_2$ are the factors that causally affect $Y$, rather than complete knowledge of the DAG.

\paragraph{Methods}  
\newcommand{\textmethodname}[1]{{\small\textsf{#1}}}
We fit linear classifiers for each dataset using the following methods:
% \begin{itemize}
\squishlist
    \item \textmethodname{Static}: a classifier trained using $\ell_2$-logistic regression without accounting for strategic adaptation.
    \item \textmethodname{DropFeatures}: a classifier trained using $\ell_2$-logistic regression without any manipulated features.
    \item \textmethodname{ManipulationProof}: a classifier that considers the agent's unconstrained best response during training, as typically done in the strategic classification literature \citep{hardt2016strategic}. % We include a formulation of the optimization problem for this approach in Appendix [ADD REF].
    \item \textmethodname{OurMethod}: a linear logistic regression classifier that results from solving the optimization program in \cref{eqn: general ERM} using the BFGS algorithm~\cite{byrd1995limited}. This model represents our approach.
\squishend
% \end{itemize}
%
\paragraph{Evaluation Criteria}
We run each method with 5-fold cross-validation and report the mean and standard deviation for each classifier on each of the following metrics:
%
%\begin{itemize}[label={}, leftmargin=0pt]
\squishlist
    \item \emph{Test Error}: the error of a classifier after training but \emph{before} decision subjects' adaptations, i.e. $\Expectation_{(x,y) \sim \Dataset} \Indicator[h(x) \neq y]$.

    \item \emph{(Worst-Case) Deployment Error}: the test error of a classifier \emph{after} decision subjects play their manipulating best response, i.e. $\Expectation_{(x,y) \sim \Dataset} \Indicator[h(\SBM) \neq y]$.

    \item \emph{(Best-Case) Improvement Rate}: the percent of improvement, defined as the proportion of the population who originally would be rejected but are accepted if they perform constructive adaptation (improving best response), i.e. $\Expectation_{(x,y) \sim \Dataset} \Indicator[h(\SBI) = +1 ~|~ y(x) = -1]$.
\squishend
%\end{itemize}

\subsection{Controlled experiments on synthetic dataset}

We perform controlled experiments using a synthetic \textds{toy} dataset to test the effectiveness of our model at incentivizing improvement in various situations. As shown in \cref{fig:causal-DAG-Z}, we set $Z_1$ and $Z_2$ as improvable features, $X_1$ and $X_2$ as their corresponding noisy proxies,  $M_1$ and $M_2$ as manipulable features, and $Y$ as the true outcome. Since we have full knowledge of this DAG structure, we can observe the changes in the true outcome after the decision subject's best response. As shown in \cref{table:toy_example_results}, Our method achieves the lowest deployment error ($20.11\%$) and improvement rate ($23.04\%$) when the model designer has full knowledge of the causal graph.

We also run experiments in which some features are \textit{misspecified}, simulating realistic scenarios in which the model designer may not be able to observe all the improvable features~\cite{haghtalab2020maximizing, shavit2020causal}, or mistakes one type of feature for another. We model these situations by changing $M_1$ into an improvable feature and $X_1$ into a manipulable feature; the results, shown in \cref{table:toy_example_results}, show that our classifier maintains a relatively high improvement rate in these cases, without sacrificing much deployment accuracy.

\begin{table}[!tb]
\definecolor{best}{HTML}{BAFFCD}
\definecolor{issue}{HTML}{FFC8BA}
\definecolor{bad}{HTML}{FFC8BA}
\newcommand{\good}[1]{\cellcolor{best}#1} 
\newcommand{\bad}[1]{\cellcolor{issue}#1} 
\newcommand{\violation}[1]{\cellcolor{bad}#1} 
\newcommand{\basic}[0]{\textsf{ST}}
\newcommand{\dropfeatures}[0]{\textsf{DF}}
\newcommand{\manipulationproof}[0]{\textsf{MP}}
\newcommand{\lighttouch}[0]{\textsf{LT}}
\newcommand{\metrics}[0]{{\cell{l}{\textit{test error} \\ \textit{deployment error} \\ \textit{improvement rate}}}}

\normalsize 
\caption{Performance metrics (mean $\pm$ standard deviation) for all methods on 4 data sets. \basic~ indicates Static, \dropfeatures~ indicates DropFeatures, \manipulationproof~ indicates ManipulationProof, and \ourmethod~ indicates our method.}
\label{table:results}
\centering\resizebox{0.8\linewidth}{!}{
\begin{tabular}{l l c c c c}
    % \normalsize 
    \toprule
        & & \multicolumn{4}{c}{\textsc{Methods}} \\
        \cmidrule(lr){3-6}
        \textbf{Dataset} & \textbf{Metrics} & \basic & \dropfeatures & \manipulationproof & \ourmethod \\
    % \midrule
    %     \sccell{c}{\textds{toy}} & \metrics & \cell{c}{\bad{$64.21$} \\ \bad{$11.54$}} & \cell{c}{$64.85$ \\ $13.13$} & \cell{c}{$75.9$ \\ $14.63$} & \cell{c}{\good{$79.39$} \\ \good{$23.49$}} \\
    \midrule
        \sccell{c}{\textds{credit}} & \metrics & \cell{c}{$29.52\pm0.37$ \\ $34.69\pm3.23$ \\ $43.70\pm2.04$} & \cell{c}{$29.66\pm0.40$ \\ $29.66\pm0.40$ \\ $40.82\pm2.81$} & \cell{c}{$29.86\pm0.52$ \\ \bad{$36.85\pm1.59$} \\ \bad{$34.62\pm0.41$}} & \cell{c}{$29.60\pm0.44$ \\ \good{$29.41\pm0.39$} \\ \good{$55.50\pm4.03$}} \\
    \midrule
        \sccell{c}{\textds{adult}} & \metrics & \cell{c}{$23.05\pm0.47$ \\ \bad{$49.15\pm7.36$} \\ \bad{$26.04\pm2.93$}} & \cell{c}{$33.55\pm0.73$ \\ \bad{$33.55\pm0.73$} \\ \good{$61.68\pm19.12$}} & \cell{c}{$24.94\pm0.52$ \\ \good{$28.62\pm1.39$} \\ $31.93\pm4.13$} & \cell{c}{$27.22\pm0.65$ \\ \good{$28.98\pm0.68$} \\ $52.07\pm6.04$} \\
    \midrule
        \sccell{c}{\textds{german}} & \metrics & \cell{c}{$30.85\pm0.82$ \\ \bad{$39.30\pm4.74$} \\ $31.70\pm5.94$ } & \cell{c}{$36.10\pm1.97$ \\ \bad{$36.10\pm1.97$} \\ $34.00\pm9.87$} & \cell{c}{$33.25\pm1.44$ \\ $37.10\pm3.70$ \\ \bad{$29.10\pm2.85$} } & \cell{c}{$34.70\pm2.15$ \\ \good{$34.15\pm2.64$} \\ \good{$53.00\pm7.81$}} \\
    \midrule
        \sccell{c}{\textds{spambase}} & \metrics & \cell{c}{$7.11\pm0.52$ \\ \bad{$38.88\pm11.37$} \\ $27.50\pm11.24$ } & \cell{c}{$10.18\pm0.45$ \\ \good{$10.18\pm0.45$} \\ \bad{$16.88\pm11.33$}} & \cell{c}{$11.52\pm0.12$ \\ $16.07\pm2.12$ \\ $18.22\pm6.04$ } & \cell{c}{$14.37\pm0.24$ \\ $14.70\pm0.46$ \\ \good{$39.84\pm8.61$}} \\
    \bottomrule
\end{tabular}
}
\end{table}
\subsection{Results}
We summarize the performance of each method in \cref{table:results}. Here are some key takeaways:

\squishlist
\item Our method produces classifiers that achieve almost the highest deployment accuracy while providing the highest percentage of improvement across all four datasets.
\item The static classifier, which does not account for adaptations, is vulnerable to strategic manipulation and consequently has the highest deployment error on every dataset.
\item Naively cutting off the manipulated features may harm the accuracy at test time -- DropFeatures incurs high test errors on \textds{Adult} ($33.55\%$) and \textds{German} ($36.10\%$).
\item The strategic classifier ManipulationProof induces the lowest improvement rates on the \textds{Credit} ($25.26\%$) and \textds{German} ($29.10\%$) datasets.
\squishend

\subsection{Effect of trade-off parameter \texorpdfstring{$\lambda$}{}}
\cref{fig:trade-off} shows the performance of linear classifiers for different values of $\lambda$ on four real datasets. Note that, since the objective function is non-convex, the trends for test error at deployment are not necessarily monotonic.
In general, we observe a trade-off between the improvement rate and deployment error: both increase as $\lambda$ increases from $0.01$ to $10$ in all four datasets.

\begin{figure}[h!]
    \centering
    \begin{subfigure}[t]{0.24\textwidth}
        %\centering
        \includegraphics[width=\linewidth]{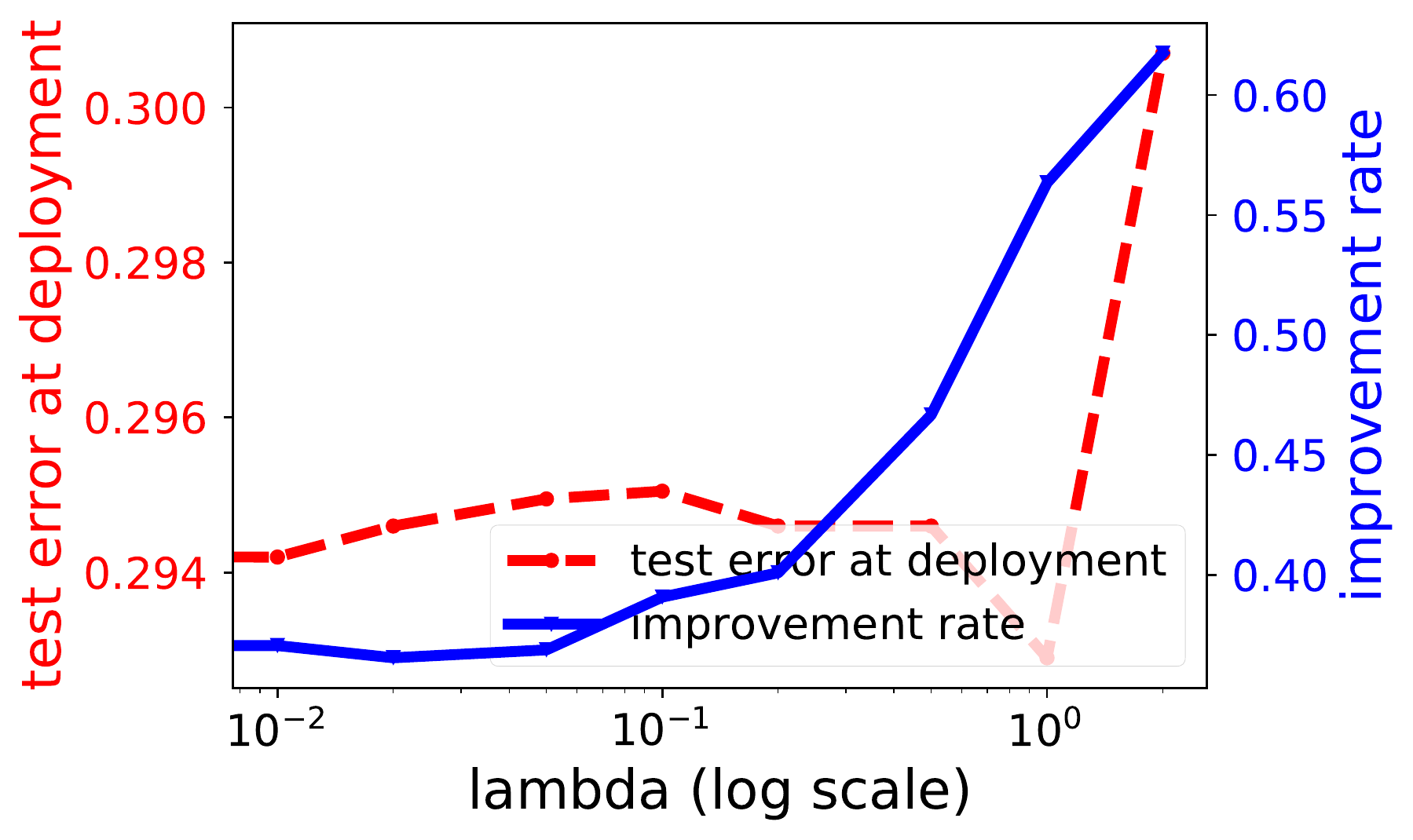}
        \caption{\textds{credit}}
    \end{subfigure}
    \begin{subfigure}[t]{0.24\textwidth}
        %\centering
        \includegraphics[width=\linewidth]{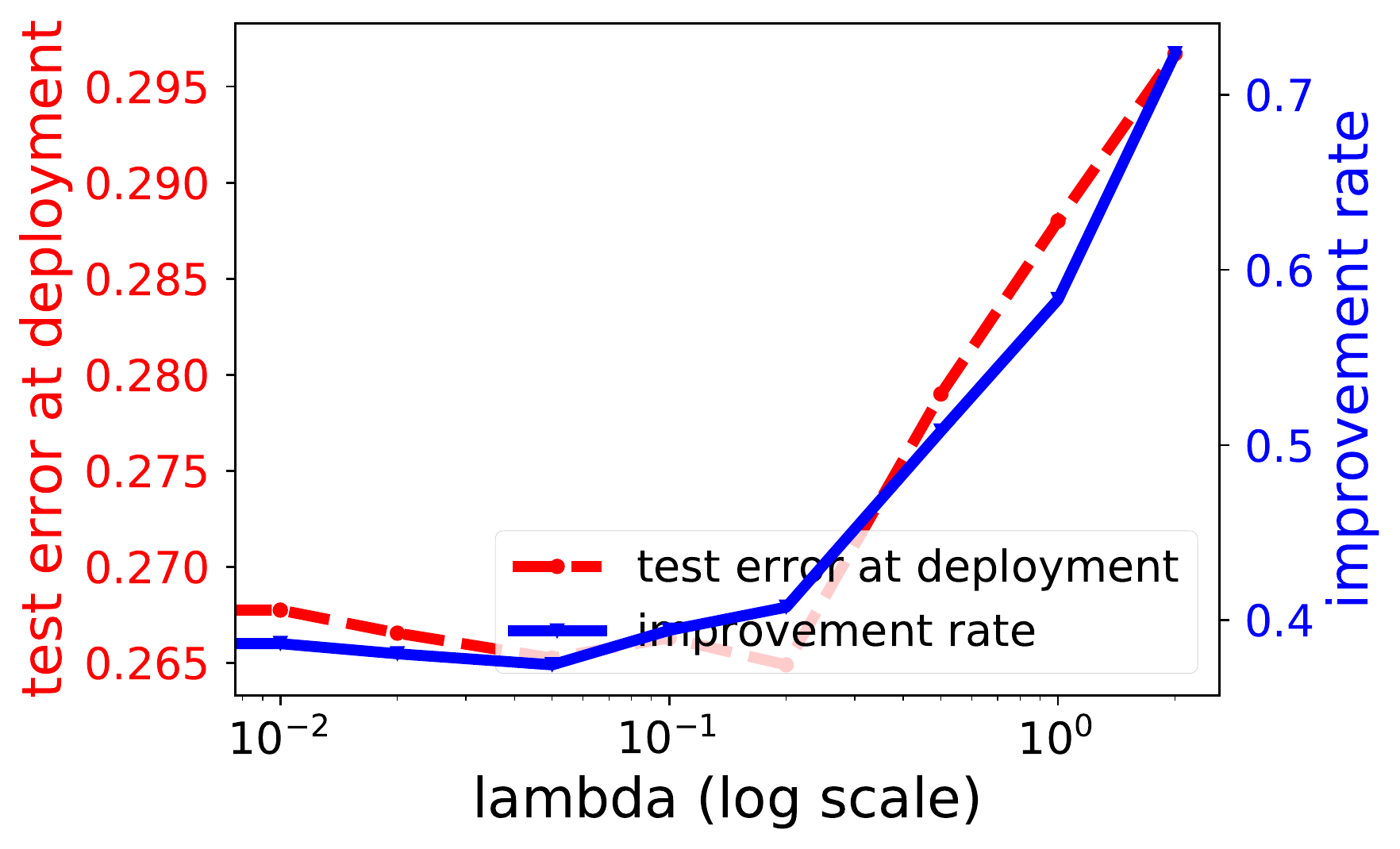}
        \caption{\textds{adult}}
    \end{subfigure}
    %\vskip\baselineskip
    %\hfill
    \begin{subfigure}[t]{0.24\textwidth}
        %\centering
        \includegraphics[width=\linewidth]{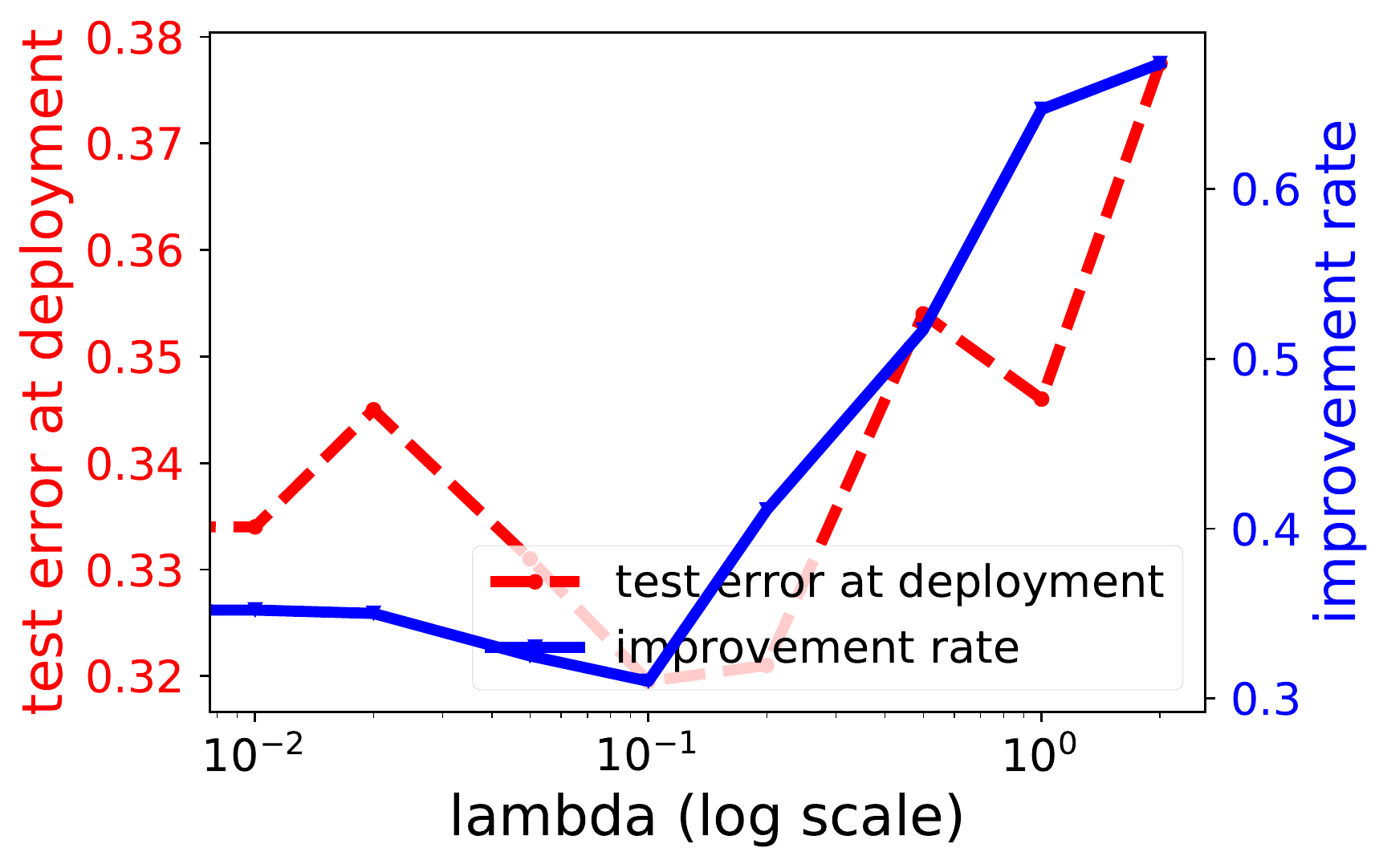}
        \caption{\textds{german}}
    \end{subfigure}
    %\hfill
    \begin{subfigure}[t]{0.24\textwidth}
        %\centering
        \includegraphics[width=\linewidth]{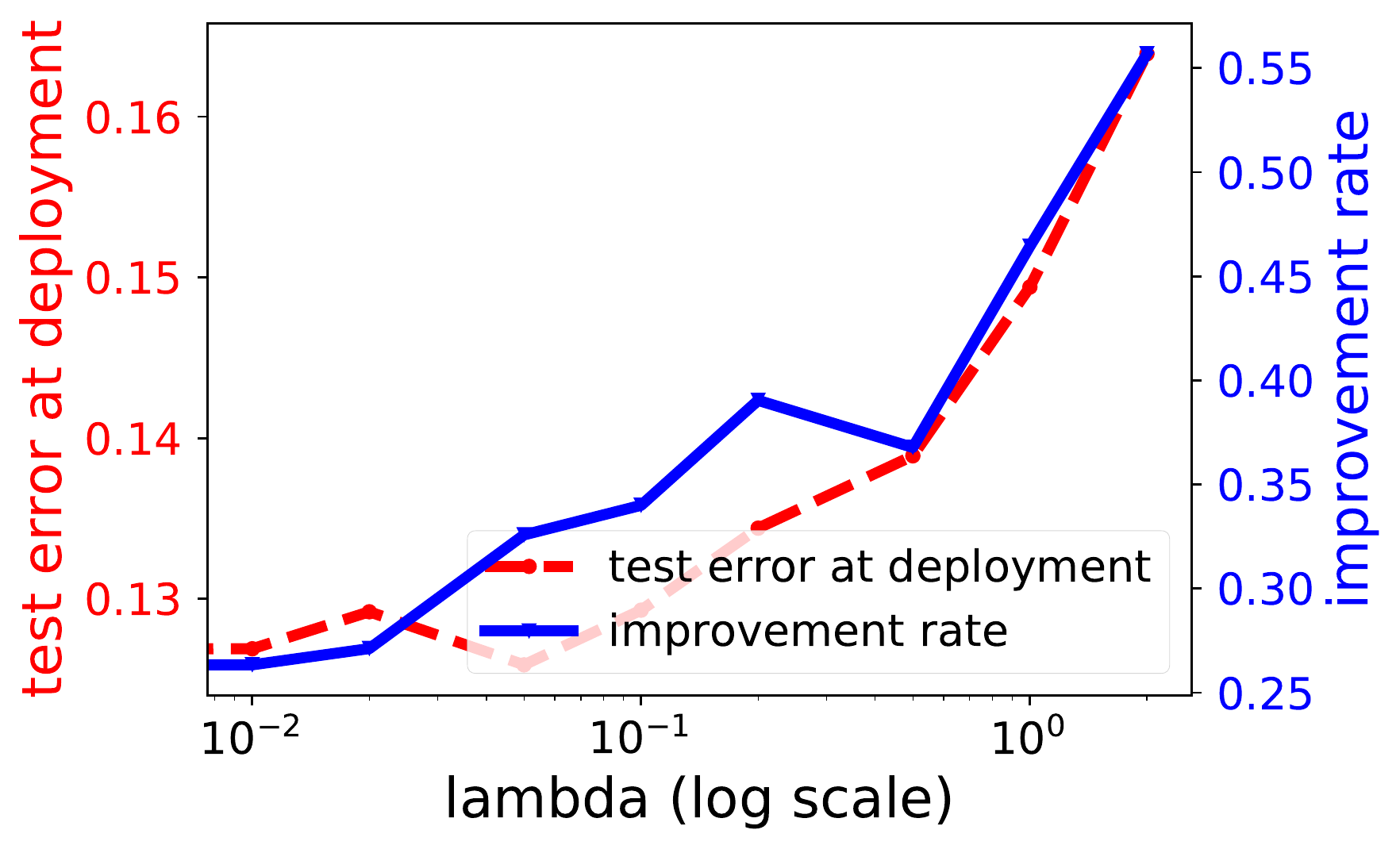}
        \caption{\textds{Spambase}}
    \end{subfigure}
    \caption{Trade-off between test error at deployment and improvement rate.}
    \label{fig:trade-off}
\end{figure}

\section{Conclusion remarks}
In this work, we study how to train a linear classifier that encourages constructive adaption. We characterize the equilibrium behavior of both the decision subjects and the model designer, and prove other formal statements about the possibilities and limits of constructive adaptation. Finally, our empirical evaluations demonstrate that classifiers trained via our method achieve favorable trade-offs between predictive accuracy and inducing constructive behavior.

Our work has several limitations:
\begin{enumerate}
    \item We assume the published classifier is linear; indeed, this is ultimately what allows for a closed-form best response (Theorem \ref{thm:best-response-close-form}) even with a relatively general cost function. However, this is clearly not true of many models actually in deployment.
    \item In order to focus on the \emph{strategic} aspects of constructive adaptation, we assume that the feature taxonomy is simply given; however, distinguishing improvable features from non-improvable features is an interesting question in its own right, and has been shown to be reducible to a nontrivial causal inference problem \cite{miller2020strategic}.
    \item Our formulation of the classification setting as a two-step process gives decision subjects only one chance to adapt their features. We suspect that extending this formalism to more rounds may create more opportunities for constructive behavior in the long term, especially for agents who cannot improve their true qualification in one round.
\end{enumerate}

\bibliographystyle{unsrt}
\bibliography{references}

% for arXiv, get ride of the checklist
%\input{checklist}

%%%%%%%%%%%%%%%%%%%%%%%%%%%%%%%%%%%%%%%%%%%%%%%%%%%%%%%%%%%%
\newpage
\appendix
\begin{center}
    \section*{\Large Appendix}
\end{center}
\section{Organization of the Appendix}
\label{sec:organization-appendix}
The Appendix is organized as follows.
\squishlist
\item Section \ref{sec:organization-appendix} provides the organization of the appendix.
\item Section \ref{sec:sec3-proof} provides the proof of \cref{thm:best-response-close-form}.
\item Section \ref{sec:discussion-proof} includes notations and proofs for the discussion in section \ref{sec:discussion}.
\item Section \ref{sec:sec4-proof} includes the proofs and derivations for section \ref{sec:method}.
\item Section \ref{sec:additional-related-work} presents additional related works.
\item Section \ref{sec:additional-experimental-result} shows additional experimental details and results, including basic information on each dataset, the computing infrastructure, and the flipsets.
\squishend

\section{Proof of Theorem \ref{thm:best-response-close-form}}
\label{sec:sec3-proof}
In this section, we provide the proof of \cref{thm:best-response-close-form}. To simplify our discussion, we focus on the unconstrained best response, i.e. the case in which $\F = \A$. The proofs for the other two types of best response ($\F = \M$, $\F = \I$) follow the same arguments.

We first prove two lemmas that allow us to reformulate the best response as an optimization problem. The first states that the decision subject's goal is to maximize their utility, but they are unwilling to pay a cost greater than 2:

\begin{lemma}[Decision Subject's Best-Response Function]
\label{lemma:best-response decision subject model}
Given a classifier $h: \mathcal X\rightarrow \{-1, +1\}$, a cost function $c: \mathcal X\times \mathcal X \rightarrow \RealNumber$, and a set of realizable feature vectors $\mathcal X^{\dag}\subseteq \mathcal X$, the \emph{best response} of a decision subject with features $x \in \mathcal{X}^\dag$ is the solution to the following optimization program:
\begin{align*}
\max_{x'\in \mathcal X^{\dag}}  \quad U(x, x')\quad \SuchThat  \quad c(x, x')\leq 2
\end{align*}
\end{lemma}

\begin{proof}
Since the classifier in our game outputs a binary decision ($-1$ or $+1$), decision subjects only have an incentive to change their features from $x$ to $x'$ when $c(x, x')\leq 2$. To see this, notice that an decision subject originally classified as $-1$ receives a default utility of $U(x,x) = f(x) - 0 = -1$ by presenting her original features $x$. Since costs are always non-negative, she can only hope to increase her utility by flipping the classifier's decision. If she changes her features to some $x'$ such that $f(x') = +1$, then the new utility will be given by
\begin{align*}
    U(x, x') = f(x') - c(x, x') = 1 - c(x, x')
\end{align*}
Hence the decision subject will only change her features if $1 - c(x, x') \geq f(x) = -1$, or $c(x,x') \leq 2$.
\end{proof}

The next lemma turns the above maximization program into a minimization program, in which the decision subject seeks the minimum-cost change in $x$ that crosses the decision boundary. If the cost exceeds 2, which is the maximum possible gain from adaptation, they would rather not modify any features.

\begin{lemma}
\label{lemma:reformulate best-response agent model}
Let $x^\star$ be an optimal solution to the following optimization problem:
\begin{align*}
    x^\star = &\argmin_{x'\in \mathcal X^*_\A(x)} \ c(x, x') \\
    \SuchThat\quad &\ \Sign(w^\T x') = 1
\end{align*}
If no solution is returned, we say an $x^\star$ such that $c(x,x^\star)=\infty$ is returned. Define $\Delta(x)$ as follows:
\begin{align*}
    \Delta(x) = 
    \begin{cases}
    x^\star,    & \text{if} \ \ c(x, x^\star)\leq 2 \\
    x,          & \text{otherwise}
    \end{cases}
\end{align*}
Then $\Delta(x)$ is an optimal solution to the optimization problem in \cref{lemma:best-response decision subject model}.
\end{lemma}

\begin{proof}
Recall that the utility function of the decision subject is $U(x, x') = f(x') - c(x, x')$, and that, by Lemma \ref{lemma:best-response decision subject model}, they will only modify their features if the utility increases, i.e. if they achieve $f(x') = +1$ and while incurring cost $c(x, x') \leq 2$.

Consider two cases for $x'\neq x$:
\begin{enumerate}
    \item When $c(x, x') > 2$, there are no feasible points for the optimization problem of \cref{lemma:best-response decision subject model}. %\yl{broken pointer}
    \item When $c(x, x')\leq 2$, we only need to consider those feature vectors $x'$ that satisfy $f(x') =1$, because if $f(x') = -1$, the decision subject with features $x$ would prefer not to change anything. Since maximizing $U(x, x') = f(x') - c(x, x')$ is equivalent to minimizing $c(x, x')$ if $f(x') = 1$, we conclude that when $c(x, x')\leq 2$, the optimum of the program of \cref{lemma:best-response decision subject model} is the same as the optimum of the program in \cref{lemma:reformulate best-response agent model}.
\end{enumerate}
\end{proof}

\cref{lemma:reformulate best-response agent model} enables us to re-formulate the objective function as follows.
%\subsubsection*{\textbf{Reformulate the objective function and constraint}}
Recall that $c(x,x') = \sqrt{(\XA - \XA')^\T S^{-1}(\XA - \XA')}$ where $S^{-1}$ is symmetric positive definite. Thus $S^{-1}$ has the following diagonalized form, in which $Q$ is an orthogonal matrix and $\Lambda^{-1}$ is a diagonal matrix:
\begin{align*}
    S^{-1} = Q^\T\Lambda^{-1}Q = (\Lambda^{-\frac{1}{2}}Q)^\T(\Lambda^{-\frac{1}{2}}Q) 
\end{align*}
With this, we can re-write the cost function as
\begin{align*}
    c(x,x') &= \sqrt{(\XA - \XA')^\T S^{-1}(\XA - \XA')}\\
            &= \sqrt{(\XA - \XA')^\T
            (\Lambda^{-\frac{1}{2}}Q)^\T
            (\Lambda^{-\frac{1}{2}}Q) 
            (\XA - \XA')}\\
            &= \sqrt{(\Lambda^{-\frac{1}{2}}Q(\XA - \XA')) ^\T (\Lambda^{-\frac{1}{2}}Q(\XA - \XA')) }\\
            &= \|\Lambda^{-\frac{1}{2}}Q(\XA - \XA')\|_2
\end{align*}
Meanwhile, the constraint in \cref{lemma:reformulate best-response agent model} can be written
    \begin{align*}
            \Sign(w\cdot x') &= \Sign(\WA\cdot \XA' + \WIM \cdot \XIM ) \\
            &= \Sign(\WA\cdot \XA' - (- \WIM \cdot \XIM)) = 1
    \end{align*}
Hence the optimization problem can be reformulated as
\begin{align}
    \min_{\XA'\in \mathcal X^*_{A}}\ & \|(\Lambda^{-\frac{1}{2}}Q(\XA - \XA'))\|_2 \label{eq: reformulate optimization problem-main}\\
    \SuchThat\ &\ \Sign(\WA\cdot \XA' -(-  \WIM \cdot \XIM)) = 1 \label{eqn:sign constraint-main}
\end{align}

The above optimization problem can be further simplified by getting rid of the $\Sign(\cdot)$:
\begin{lemma}
\label{lemma:replace inequlity constraint with equality constraint}
If $\XA^\mp$ is an optimal solution to \cref{eq: reformulate optimization problem-main} under constraint \cref{eqn:sign constraint-main}, then it must satisfy $\WA\cdot \XA^\mp - (- \WIM\cdot \XIM) = 0$.
\end{lemma}

\begin{proof}
We prove by contradiction. Let $x^\mp_\A$ is an optimal solution to \cref{eq: reformulate optimization problem-main} and suppose towards contraction that $\WA x^\mp_\A > - \WIM\cdot \XIM$. Since the original feature vector $x$ was classified as $-1$, we have
\begin{align*}
    \WA \cdot x^\mp_\A &>  - \WIM\cdot \XIM,~~~\WA \cdot \XA <  -\WIM\cdot \XIM
\end{align*}
By the continuity properties of linear vector space, there exists $\mu\in (0,1)$ such that: 
\begin{align*}
    \WA \left(\mu\cdot \XA^{\mp} + (1-\mu) \XA\right) = -\WIM\cdot \XIM
\end{align*}
Let $\XA'' = \mu\cdot \XA^{\mp} + (1-\mu) \XA$. Then $\Sign(\WA \XA'' - ( - \WIM\cdot \XIM)) = 1$, i.e., $x''_\A$ also satisfies the constraint.
Since $\XA^{\mp}$ is an optimum of \cref{eq: reformulate optimization problem-main}, we have %for all $\XA''\in \mathcal X^*_{A}$: %we should have 
\begin{align*}
    \|\Sigma^{-\frac{1}{2}}Q (\XA^{\mp} - \XA)\|\leq \|\Sigma^{-\frac{1}{2}}Q (\XA'' - \XA)\|
\end{align*}
However, we also have: %when $x''_A = \mu\cdot \XA^{\mp} + (1-\mu) \XA$, we have:
\begin{align*}
\|\Sigma^{-\frac{1}{2}}Q (\XA'' - \XA)\|
&=\|\Sigma^{-\frac{1}{2}}Q (\mu\cdot \XA^{\mp} + (1-\mu) \XA - \XA)\|\\
&=
    \|\Sigma^{-\frac{1}{2}}Q (\mu\cdot (\XA^{\mp}-\XA))\|\\
    &=\mu \|\Sigma^{-\frac{1}{2}}Q (\XA^\mp - \XA)\|\\
    & < \|\Sigma^{-\frac{1}{2}}Q (\XA^\mp - \XA)\|
\end{align*}
contradicting our assumption that $x^\mp_\A$ is optimal. Therefore $x^\mp_\A$ must satisfy $\WA x^\mp_\A = -\WIM\cdot \XIM$.
\end{proof}

As a result of Lemma \ref{lemma:replace inequlity constraint with equality constraint}, we can replace the constraint in \cref{eq: reformulate optimization problem-main} with its corresponding equality constraint without changing the optimal solution.\footnote{A similar argument was made by \cite{haghtalab2020maximizing} but here we provide a proof for a more general case, where the objective function is to minimize a weighted norm instead of simply $\|\XA - \XA'\|_2$.} The decision subject's best-response program from \cref{lemma:best-response decision subject model} is therefore equivalent to
\begin{align}
    \min_{\XA'\in \mathcal X^*_{A}}\ & \|(\Lambda^{-\frac{1}{2}}Q(\XA - \XA'))\|_2 \label{OPT:1} \\
    \SuchThat\ &\ \WA\cdot \XA' -( -  \WIM \cdot \XIM) = 0 \label{OPT:2} 
\end{align}

%\subsubsection*{Putting Things Together}~~\\
% In \cref{sec:agent}, we show the agent's best response optimization problem is equivalent to the following one:
% \begin{align}
%     \min_{\XA'}\ & \|(\Lambda^{-\frac{1}{2}}Q(\XA - \XA'))\|_2 \label{OPT:1} \\
%     \SuchThat\ &\ \WA\cdot \XA' -( b -  w_U \cdot x_U) = 0 \label{OPT:2}
% \end{align}
The following lemma gives us a closed-form solution for the above optimization problem:
\begin{lemma}
\label{lemma:norm minimization with equality constraints}
The optimal solution to the optimization problem defined in \cref{OPT:1} and \cref{OPT:2}

has the following closed form:
$$
\XA^\mp = \XA - \frac{w^{\T}x}{\WA^\T S \WA}S \WA.
$$
\end{lemma}
\begin{proof}
%\yc{relate it with $F$ best response}\\
Notice that the above program has the form
    \begin{align*}
        \min_{\XA'\in \mathcal \XA^*}\ &\|A\XA'-b\|_2 \\
        \SuchThat\ & \  C\XA' = d
    \end{align*}
where $A = \Lambda^{-\frac{1}{2}}Q$, $b = \Lambda^{-\frac{1}{2}}Q \XA$, $C = \WA^{\T}$, and $d =  - \WIM^{\T} \XIM$. Note the following useful equalities:
    \begin{align*}
        A^\T A &= (\Lambda^{-\frac{1}{2}}Q)^\T\Lambda^{-\frac{1}{2}}Q = S^{-1}\\
        (A^\T A)^{-1} &= S\\
        A^\T b &= (\Lambda^{-\frac{1}{2}}Q)^\T \Lambda^{-\frac{1}{2}}Q\XA = S^{-1}\XA
    \end{align*}
The above is a norm minimization problem with equality constraints, whose optimum $\XA^\mp$ has the following closed form \cite{boyd2004convex}:
    \begin{align*}
        \XA^\mp &= (A^\T A)^{-1}
        \left( A^\T b - C^\T(C(A^\T A)^{-1}C^\T)^{-1}(C(A^\T A)^{-1}A^\T b - d)  \right)\\
        &= S \left( S^{-1}\XA -  \WA ( \WA^\T S \WA )^{-1} (\WA^\T S (S^{-1}\XA) - (- \WIM^{\T} \XIM) )\right)\\
        & = \XA - S\left(
        \WA(\WA^{\T} S \WA)^{-1}(\WA^{\T} \XA + \WIM^{\T} \XIM )
        \right)\\
        &= \XA - \frac{w^{\T}x }{\WA^{\T} S \WA}S \WA
    \end{align*}
\end{proof}
We can now compute the cost incurred by an individual with features $x$ who plays their best response $x^\mp$:
    \begin{align*}
        c(x,x\mp) &= \sqrt{(\XA - \XA^\mp)^{\T} S^{-1} (\XA-\XA^\mp)}\\
        &= \sqrt{\left(\frac{w^{\T}x}{\WA^\T S \WA}S \WA\right)^{\T} 
        S^{-1} 
        \left(\frac{w^\T x}{\WA^{\T} S \WA}S \WA\right)}\\
        &=\frac{|w^{\T}x|}{\sqrt{\WA^{\T} S \WA}}
    \end{align*}
Hence an decision subject who was classified as $-1$ with feature vector $x$ has the unconstrained best response
\begin{align*}
%\label{eq:best_response_delta_x}
    \Delta(x) =
    \begin{cases}
    x, & \text{if } \frac{|w^{\T}x|}{\sqrt{\WA^{\T} S \WA}}\geq 2\\
    \left[\XA - \frac{w^{\T}x}{\WA^{\T} S\WA}S\WA ~|~ \XIM\right], & \text{otherwise}
    \end{cases}
\end{align*}
% Similarly, finding $\BI(x)$ for any $x$ is equivalent to solving the following optimization problem:
% \begin{align*}
%     \min_{x'\in \mathcal X^*_I(x)} &\ c(x, x') \\
%     \SuchThat \quad &\ \Sign(w^{\T} x') = 1
% \end{align*}
% If we let $x' = [\XI' ~|~ \XM ~|~ \XIM]$, then this can be re-written as
% \begin{align*}
% \min_{\XI'\in \mathcal \XI^*} &\ \sqrt{(\XI-\XI')^\T\SI^{-1}(\XI - \XI')}\\
%   \SuchThat \ \ & \WI\cdot \XI' =  - \WM\cdot \XM - \WIM \cdot \XIM
% \end{align*}
% By the same argument as before, we have the closed-form solution
% \begin{align*}
%     \XI' = \XI - \frac{w^{\T} x }{\WI^\T \SI\WI} \SI^{-1} \WI
% \end{align*}
% whose cost is
% \begin{align*}
%     c(x,x')
%     = \sqrt{(\XI - \XI')^\T\SI^{-1}(\XI - \XI')}
%     = \frac{|w^{\T} x|}{\sqrt{\WI^\T S_I \WI}}
% \end{align*}
% This yields the improving best response
% \begin{align*}
%     \BI(x) = 
%     \begin{cases}
%     x,   & \text{if } \frac{|w^{\T} x |}{\sqrt{\WI^\T \SI \WI}} \geq 2 \\
%     \left[\XI - \frac{w\cdot x }{\WI^{\T} \SI\WI} \SI \WI ~|~ \XM ~|~ \XIM \right], & \text{otherwise}
%     \end{cases}
% \end{align*}
which completes the proof of \cref{thm:best-response-close-form}.

\section{Proofs of Propositions in Section \ref{sec:discussion}}
\label{sec:discussion-proof}
\paragraph{Notation}
We make use of the following additional notation:
\begin{itemize}
    \item $v^{(i)}$ denotes the $i$-th element of a vector $v$
    \item For any $\textsf{F} \in \{\A, \I, \M\}$, $\Delta^\textsf{F} \in \RealNumber^{d_\textsf{F}}$ denotes the vector containing only features of type \textsf{F} within the best response $\Delta(x)$.
    \item \textbf{0} denotes the vector whose elements are all 0
    \item $A \succ B$ indicates that matrix $A - B$ is positive definite
    \item $e_i$ denotes the vector containing $1$ in its $i$-th component and $0$ elsewhere
\end{itemize}

\subsection{Proof of Proposition \ref{prop:1}}

\begin{proof}
Let $w_\M^{(m)} \neq 0$, and consider an decision subject with original features $x$ who was classified as $-1$. By Theorem \ref{thm:best-response-close-form}, the actionable sub-vector of $x$'s unconstrained best response is
\begin{align*}
    \Delta^\A(x)
    = \frac{w^\T x}{\WA^\T S \WA}S\cdot \WA
    = \frac{w^\T x}{\WA^\T S \WA}
        \begin{bmatrix}
            \SI & 0\\
            0 & \SM
        \end{bmatrix}
        \begin{bmatrix}
        \WI\\
        \WM
        \end{bmatrix}
    = \frac{w^{\T}x}{\WA^{\T} S \WA}
        \begin{bmatrix}
            \SI\cdot \WI\\
            \SM\cdot \WM
        \end{bmatrix}
\end{align*}
And in particular,
\begin{align*}
    \Delta^\M(x)
    = \frac{w^{\T}x}{\WA^{\T} S \WA} \SM \cdot \WM
\end{align*}

Since $x$ was initially classified as $-1$, we have $w^{\T}x < 0$, which means $\frac{w^{\T}x}{\WA S \WA}\neq 0$. For convenience, let $c = \frac{w^{\T}x}{\WA S \WA}$. We have
\begin{align*}
    \Delta^\M(x) - \XM
    = c\SM\WM - \XM
    = \SM(c\WM - \SM^{-1} \XM)
\end{align*}
Now examine the following:
\begin{align*}
    (c\WM - \SM^{-1} \XM)^{(m)}
    &= cw_\M^{(m)} - (S_\M^{-1} \XM)^{(m)} \\
    &= cw_\M^{(m)} - \sum_{i=1}^{d_\M} (S_\M^{-1})^{(im)} \XM^{(m)}
\end{align*}
Recall that $c w_\M^{(m)} \neq 0$. Hence if $\sum_{i=1}^{d_\M} (S_\M^{-1})^{(im)} = 0$, or if
\[ x_\M^{(m)} \neq \frac{c w_\M^{(m)}}{\sum_{i=1}^{d_\M} (S_\M^{-1})^{(im)}} ,\]
then $(c\WM - \SM^{-1} \XM)^{(m)} \neq 0$, and therefore $c\WM - S_\M^{-1} \XM \neq \mathbf{0}$. Since $\SM$ is positive definite, it has full rank, which implies
\[ \Delta^\M(x) - \XM = \SM(c\WM - S_\M^{-1} \XM) \neq \mathbf{0} \]
as required. With this, we have shown that when there exists a manipulated feature $x^{(m)}$ whose corresponding coefficient $\WA^{(m)}\neq 0$, the classifier is vulnerable to changes in the manipulated features by the vast majority of decision subjects.
\end{proof}

\subsection{Proof of Proposition \ref{prop:2}}
\begin{proof}
Consider a decision subject with features $x$ such that $h(x) = -1$. Suppose $x$ can flip this classification result by performing the improving best response $\BI(x)$, which implies that the cost of that action is no greater than $2$ for this decision subject. We therefore have:
\begin{align*}
    2
    \geq c(x, \BI(x))
    = \frac{|w^{\T} x|}{\sqrt{\WI^\T\SI\WI}}
    > \frac{|w^{\T} x |}{\sqrt{\WI^\T\SI\WI + \WM^\T\SM\WM}}
    = \frac{|w^{\T} x|}{\sqrt{\WA^{\T}S\WA}}
    = c(x, \Delta(x))
\end{align*}
where the strict inequality is due to the fact that $\SM \succ 0$ and $\WM \neq \textbf{0}$. As we have shown that $c(x, \Delta(x)) < 2$, we conclude whenever an decision subject can successfully flip her decision by the improving best response, she can also achieve it by performing the unconstrained best response.

On the other hand, consider the case when the unconstrained best response of a decision subject with features $x^*$ has cost exactly 2:
\begin{align*}
    2
    = c(x^*, \Delta(x^*))
    = \frac{|w^{\T} x^* |}{\sqrt{\WA^{\T} S \WA}}
    = \frac{|w^{\T} x^*|}{\sqrt{\WI^{\T} \SI\WI + \WM^{\T}\SM\WM}}
    < \frac{|w^{\T} x^*|}{\sqrt{\WI^\T\SI\WI}}
    = c(x^*, \BI(x^*))
\end{align*}
where the strict inequality is due to the fact that $\SM \succ 0$ and $\WM \neq \textbf{0}$. As we have shown that $c(x^*, \BI(x^*)) > 2$, we conclude that while the unconstrained best response is viable for this decision subject, the improving best response is not.
\end{proof}

\subsection{Proof of Proposition \ref{prop:3}}

\begin{proof}
Consider any cost matrix $S^{-1} \in \RealNumber^{\DA \times \DA}$ and any nontrivial classifier $h$ (i.e. $h$ does not assign every $x$ the same prediction). Since $S^{-1}$ is positive definite, so is its inverse $S$, and all of their diagonal entries are positive. And since $h$ is nontrivial, it must contain a nonzero coefficient $w_i \neq 0$. Additionally, let $w_j$ be any other coefficient.

Let $\tilde{S}^{-1} = S^{-1} + \tau (e_i e_j^\T + e_j e_i^\T)$ for some constant $\tau \in \RealNumber$ to be set later. We claim that there exists $\tau$ such that the best-response adaptation always costs less under $\tilde{S}^{-1}$ than $S^{-1}$. To do so, we compute the inverse of $\tilde{S}^{-1}$ and invoke the closed-form cost expression given by Theorem \ref{thm:best-response-close-form}.

To begin computing the inverse, note that by the Sherman-Morrison-Woodbury formula \cite{golub2013matrix},
\begin{align}
    \tilde{S} = \left(\tilde{S}^{-1}\right)^{-1}
    &= S - \tau S
        \begin{bmatrix}
        e_i & e_j
        \end{bmatrix} 
        \left(I + \tau\begin{bmatrix}
        e_j^\T \\
        e_i^\T
        \end{bmatrix} 
        S
        \begin{bmatrix}
        e_i & e_j
        \end{bmatrix}
        \right)^{-1}
        \begin{bmatrix}
        e_j^\T \\
        e_i^\T
        \end{bmatrix} 
        S\\
    &= S - \tau S
        \begin{bmatrix}
        e_i &  e_j
        \end{bmatrix} 
        \left(I +\tau
        \begin{bmatrix}
        S_{ij} & S_{jj} \\
        S_{ii} & S_{ij}
        \end{bmatrix}
        \right)^{-1}
        \begin{bmatrix}
        e_j^\T \\
        e_i^\T
        \end{bmatrix} S\\
    &= S - \tau S
        \begin{bmatrix}
        e_i &  e_j
        \end{bmatrix} 
        \left[
            \tau \left(\frac{1}{\tau} I +
            \begin{bmatrix}
            S_{ij} & S_{jj} \\
            S_{ii} & S_{ij}
            \end{bmatrix}
            \right)
        \right]^{-1}
        \begin{bmatrix}
        e_j^\T \\
        e_i^\T
        \end{bmatrix} S\\
    &= S - \tau S
        \begin{bmatrix}
        e_i &  e_j
        \end{bmatrix}
        \tau^{-1}
        \begin{bmatrix}
        \frac{1}{\tau} + S_{ij} & S_{jj} \\
        S_{ii} & \frac{1}{\tau} + S_{ij}
        \end{bmatrix}^{-1}
        \begin{bmatrix}
        e_j^\T \\
        e_i^\T
        \end{bmatrix} S\\
    &= S - S
        \begin{bmatrix}
        e_i &  e_j
        \end{bmatrix}
        {\underbrace{
            \begin{bmatrix}
            \frac{1}{\tau} + S_{ij} & S_{jj} \\
            S_{ii} & \frac{1}{\tau} + S_{ij}
            \end{bmatrix}
        }_{T}}^{-1}
        \begin{bmatrix}
        e_j^\T \\
        e_i^\T
        \end{bmatrix} S
        \label{equation:tildeS-in-terms-of-T}
\end{align}
Clearly, we can ensure that $T$ is invertible by setting $\tau$ so that $\det(T) \neq 0$. But as the following lemmas show, we can actually say much more: $\det(T)$ can be made either positive or negative, and moreover, both can be accomplished with a choice of $\tau > 0$ or $\tau < 0$. This flexibility in choosing $\tau$ will become crucial later.

First, we need the following useful fact about positive definite matrices:

\begin{lemma}[Off-diagonal entries of a positive definite matrix]
\label{lemma:off-diagonal-entries-of-a-positive-definite-matrix}
If $A \in \RealNumber^{n \times n}$ is symmetric positive definite, then for all $i, j \in [n]$, $\sqrt{A_{ii} A_{jj}} > |A_{ij}|$.
\end{lemma}
\begin{proof}
By positive definiteness, we have, for any nonzero $\alpha, \beta \in \RealNumber$,
\begin{align*}
    (\alpha e_i + \beta e_j)^\T A (\alpha e_i + \beta e_j)
    = \alpha^2 A_{ii} + \beta^2 A_{jj} + 2\alpha \beta A_{ij}
    > 0
\end{align*}
For a choice of $\alpha = -A_{ij}$ and $\beta = A_{ii}$, we have
\begin{align*}
    A_{ij}^2 A_{ii} + A_{ii}^2 A_{jj} - 2 A_{ij}^2 A_{ii}
    = A_{ii} (A_{ii} A_{jj} - A_{ij}^2)
    > 0
\end{align*}
Since $A_{ii} > 0$, we must have $A_{ii} A_{jj} - A_{ij}^2 > 0$, from which the claim follows.
\end{proof}

Now we can characterize the possible settings of $\tau$ and $\det(T)$:

\begin{lemma}[Possible settings of $\tau$]
\label{lemma:possible-settings-of-tau}
There exist $\tau_{\mathrm{max}}, \tau_{\mathrm{min}} > 0$ such that the following hold:
\begin{enumerate}
    \item $\det(T) > 0$ for any $\tau \in \RealNumber$ such that $\tau_{\mathrm{max}} \geq |\tau| > 0$.
    \item $\det(T) < 0$ for any $\tau \in \RealNumber$ such that $\tau_{\mathrm{min}} \leq |\tau|$.
\end{enumerate}
\end{lemma}

\begin{proof}
To prove the first claim, note that having
\begin{align*}
    \det(T)
    = \left(\frac{1}{\tau} + S_{ij}\right)^2 - S_{ii} S_{jj}
    > 0
\end{align*}
is equivalent to
\begin{align*}
    \left| \frac{1}{\tau} + S_{ij} \right|
    > \sqrt{S_{ii} S_{jj}}
\end{align*}
It suffices to choose $\tau$ such that
\begin{align*}
    \left| \frac{1}{\tau} \right| - \left| S_{ij} \right|
        &> \sqrt{S_{ii} S_{jj}} \\
    \frac{1}{\left|\tau\right|}
        &> \sqrt{S_{ii} S_{jj}} + |S_{ij}|
\end{align*}
So any $\tau$ such that $0 < |\tau| < \left(\sqrt{S_{ii} S_{jj}} + |S_{ij}|\right)^{-1}$ results in $\det(T) > 0$. Analogously, for the second claim, a sufficient condition for $\det(T) < 0$ is that
\begin{align*}
    \frac{1}{\left|\tau\right|}
    < \sqrt{S_{ii} S_{jj}} - |S_{ij}|
\end{align*}
By Lemma \ref{lemma:off-diagonal-entries-of-a-positive-definite-matrix}, the right-hand side is positive. Hence it suffices to pick any $\tau$ such that
\begin{align*}
    |\tau| > \left( \sqrt{S_{ii} S_{jj}} - |S_{ij}| \right)^{-1}.
\end{align*}
\end{proof}

With this lemma in place, we can describe the difference between the inverses of $S^{-1}$ and $\tilde{S}^{-1}$. Denote this matrix by $E = S - \tilde{S}$. We show the following:

\begin{lemma}[Difference between inverse cost matrices]
\label{lemma:E-matrix}
The $k,\ell$-th entry of $E$ has the following form:
\begin{align*}
    E_{k\ell}
    = \frac{1}{\det(T)} \left(E_{k\ell}' + \frac{1}{\tau} E_{k\ell}''\right)
\end{align*}
where $E_{k\ell}'$ and $E_{k\ell}''$ do not depend on $\tau$.
\end{lemma}
\begin{proof}
Assume that $\tau$ has been chosen so that $\det(T) \neq 0$, as Lemma \ref{lemma:possible-settings-of-tau} showed to be possible. We then have
\begin{align*}
    T^{-1}
    = \frac{1}{\det(T)}
        \begin{bmatrix}
        \frac{1}{\tau} + S_{ij} & -S_{jj} \\
        -S_{ii} & \frac{1}{\tau} + S_{ij}
        \end{bmatrix}
\end{align*}
Thus continuing from equation \ref{equation:tildeS-in-terms-of-T}, we have
\begin{align*}
    \tilde{S}
    &= S - \frac{1}{\det(T)} S
        \underbrace{\begin{bmatrix}
        e_i &  e_j
        \end{bmatrix}
        \begin{bmatrix}
        \frac{1}{\tau} + S_{ij} & -S_{jj} \\
        -S_{ii} & \frac{1}{\tau} + S_{ij}
        \end{bmatrix}
        \begin{bmatrix}
        e_j^\T \\
        e_i^\T
        \end{bmatrix}}_{V}
        S
\end{align*}
It can be verified that $V$ is a $\DA \times \DA$ matrix whose only nonzero entries are
\begin{align*}
    V_{ii}          = -S_{jj}, \qquad
    V_{jj}          = -S_{ii}, \qquad
    V_{ij} = V_{ji} = \frac{1}{\tau} + S_{ij}
\end{align*}

Next we evaluate the $\DA \times \DA$ matrix $SVS$. For any $k,\ell \in [\DA]$, we have
\begin{align*}
    (SVS)_{k\ell}
    &= \sum_{i'=1}^\DA \sum_{j'=1}^\DA S_{k i'} V_{i'j'} S_{j' \ell} \\
    &= S_{ki} V_{ii} S_{i\ell}
        + S_{ki} V_{ij} S_{j\ell}
        + S_{kj} V_{ji} S_{i\ell}
        + S_{kj} V_{jj} S_{j\ell}
        && \text{($V$ has four nonzero entries)} \\
    &= V_{ii} S_{ki} S_{i\ell}
        + V_{jj} S_{kj} S_{j\ell}
        + V_{ij} (S_{ki} S_{j\ell} + S_{kj} S_{i\ell})
        && \text{($V_{ij} = V_{ji}$)} \\
    &= -S_{jj} S_{ki} S_{i\ell}
        - S_{ii} S_{kj} S_{j\ell}
        + \left(\frac{1}{\tau} + S_{ij}\right) (S_{ki} S_{j\ell} + S_{kj} S_{i\ell}) \\
    &= \underbrace{-S_{jj} S_{ki} S_{i\ell}
        - S_{ii} S_{kj} S_{j\ell}
        + S_{ij} (S_{ki} S_{j\ell} + S_{kj} S_{i\ell})
        }_{E_{k\ell}'}
        + \frac{1}{\tau} \underbrace{(S_{ki} S_{j\ell} + S_{kj} S_{i\ell})}_{E_{k\ell}''}
\end{align*}
which proves the claim.
\end{proof}

We now compute the marginal best-response cost incurred due to the difference between the inverse cost matrices, $E = S - \tilde{S}$. We have
\begin{align*}
    \WA^\T E \WA
    &= \sum_{k=1}^\DA \sum_{\ell=1}^\DA w_k w_\ell E_{k\ell} \\
    &= \frac{1}{\det(T)} \sum_{k=1}^\DA \sum_{\ell=1}^\DA w_k w_\ell \left(E_{k\ell}' + \frac{1}{\tau} E_{k\ell}''\right)
        && \text{(by Lemma \ref{lemma:E-matrix})} \\
    &= \frac{1}{\det(T)} \left[
        \underbrace{\sum_{k=1}^\DA \sum_{\ell=1}^\DA w_k w_\ell E_{k\ell}'}_{E'}
        + \frac{1}{\tau} \underbrace{\sum_{k=1}^\DA \sum_{\ell=1}^\DA w_k w_\ell E_{k\ell}''}_{E''}
        \right]
\end{align*}
By Lemma \ref{lemma:possible-settings-of-tau}, there exists $\tau \neq 0$ such that
\begin{align*}
    \Sign(\det(T)) = -\Sign(E')
    \quad \mathrm{and} \quad
    \Sign(\tau) = -\Sign(\det(T)) \cdot \Sign(E'')
\end{align*}
Such a choice of $\tau$ results in $\WA^\T E \WA < 0$. Finally by Theorem \ref{thm:best-response-close-form}, we have for all $x$ that
\begin{align*}
    c_{\tilde{S}^{-1}}(x, \Delta_{\tilde{S}^{-1}}(x))
    = \frac{|w^\T x|}{\sqrt{\WA^\T \tilde{S} \WA}}
    = \frac{|w^\T x|}{\sqrt{\WA^\T S\WA - \WA^\T  E \WA}}
    < \frac{|w^\T x|}{\sqrt{\WA^\T S \WA}}
    = c_{S^{-1}}(x, \Delta_{S^{-1}}(x))
\end{align*}
which completes the proof.
\end{proof}
\newpage

\subsection{Proof of Proposition \ref{prop:4}}
\begin{proof}
Let the cost covariance matrices for groups $\Phi$ and $\Psi$ be
\begin{align*}
    S_{\Psi}^{-1} = \begin{bmatrix}
        S_{\I}^{-1} & 0 \\
        0 & S_{\M,\Phi}^{-1} \\
    \end{bmatrix}, 
    \quad \quad
    S_{\Phi}^{-1} = \begin{bmatrix}
        S_{\I}^{-1} & 0 \\
        0 & S_{\M,\Psi}^{-1} \\
    \end{bmatrix}
\end{align*}
Here, we see that both groups have the same cost of changing improvable features, as represented in the cost submatrix $\SI^{-1}$. However, the cost of manipulation for group $\Phi$ is higher than that of group $\Psi$, namely $S_{\M,\Phi}^{-1}\succ S_{\M,\Psi}^{-1}$.

We are now equipped to compare the costs for the two decision subjects:
\begin{align*}
    c(x_\phi, \Delta(x_\phi)) &= \frac{|w^{\T} x_\phi|}{\sqrt{\WA^{\T} S_\Phi \WA}}
                    = \frac{|w^{\T} x|}{\sqrt{\WI^{\T} \SI \WI + \WM^{\T}\cdot  S_{\M,\Phi} \cdot \WM}}\\
    c(x_\psi, \Delta(x_\psi)) &= \frac{|w^{\T} x_\psi|}{\sqrt{\WA^{\T} S_\Psi \WA}}
                    = \frac{|w^{\T} x|}{\sqrt{\WI^{\T}
                    \SI \WI + \WM^{\T} \cdot  S_{\M,\Psi} \cdot \WM}}
\end{align*}
Since $S^{-1}_{\M,\Phi}\succ S^{-1}_{\M,\Psi}$, we have $S_{\M,\Phi}\prec S_{\M,\Psi}$. And since $\WM\neq \textbf{0}$, this implies $0 < \WM^{\T} S_{\M,\Phi} \WM < \WM^{\T} \cdot S_{\M,\Psi} \cdot \WM$. As a result, $c(x_\phi, \Delta(x_\phi)) > c(x_\psi, \Delta(x_\psi))$ as required.
\end{proof}

\section{Proofs and Derivations in Section \ref{sec:method}}
\label{sec:sec4-proof}

\subsection{Proof of Proposition \ref{prop:manipulation-risk}}
\label{subsec:proof-of-prop-manipulation-risk}
\begin{proof}
We want to show that the standard strategic risk conditioned on an unchanged true label is upper-bounded by the first term in our model designer's objective, $R_\M(h)$:
\begin{align*}
    \Expectation_{x\sim\Dataset} \left[\Indicator[h(\SB)\neq y] ~|~ \Delta(y) = y\right]
    \leq \Expectation_{x\sim \Dataset} \left[\Indicator(h(\SBM) \neq y)\right]
\end{align*}
We assume that the manipulating best response is more likely to result in a positive prediction than the unconstrained best response, given that the true labels do not change:
\begin{align} \label{equation:manipulation-risk-assumption}
    \Expectation_{x\sim \Dataset} \left[\Indicator[h(\SB) \neq y] ~|~ \Delta(y) = y\right]
    \leq \Expectation_{\Dataset} \left[\Indicator[h(\SBM) \neq y] ~|~ \Delta_{\M}(y) = y\right] 
\end{align}

We therefore have:
\begin{align*}
    & \Expectation_{x\sim \Dataset} \left[\Indicator(h(\SBM) \neq y)\right] \\
    =& \Expectation_{x\sim \Dataset} \left[\Indicator(h(\SBM) \neq y) ~|~ \Delta_\M(y) \neq y\right] \cdot \Probability[\Delta_\M(y) \neq y] \\
        &\quad \quad + \Expectation_{x\sim \Dataset} \left[\Indicator(h(\SBM) \neq y) ~|~ \Delta_\M(y)= y\right] \cdot \Probability[\Delta_\M(y) = y] \\
    =& \Expectation_{x\sim \Dataset} \left[\Indicator(h(\SBM) \neq y) ~|~ \Delta_\M(y) = y\right]
        && \text{($\Probability[\Delta_\M(y) = y] = 1$)} \\
    \geq& \Expectation_{x\sim \Dataset} \left[\Indicator(h(\SB) \neq y) ~|~ \Delta(y) = y\right]
        && \text{(by equation \ref{equation:manipulation-risk-assumption})}
\end{align*}
\end{proof}

\subsection{Proof of Proposition \ref{prop:improvement-risk}}
\label{subsec:proof-of-prop-improvement-risk}
\begin{proof}
Let $\Dataset^*$ be the distribution induced by deploying classifier $h$. By the covariate shift assumption, $\Probability_{\Dataset^*}(Y=y|X=x) = \Probability_{\Dataset}(Y=y|X=x)$. Therefore
\begin{align*}
    \Pr_{x\sim \Dataset^*} [y(x) = +1]
    =& \Expectation_{\Dataset^*}[\Indicator[y(x)= +1]]\\
    =& \int \Indicator[y(x)= +1] \Probability_{\Dataset^*}(X = x)dx\\
    =& \int \Indicator[y(x)= +1] \frac{\Probability_{\Dataset^*}(X = x)}{\Probability_{D}(X = x)} \Probability_{\Dataset}(X = x) dx\\
    =& \int \Indicator[y(x)= +1] \omega_h(x) \Probability_{\Dataset}(X = x) dx\\
    =& \Expectation_{\Dataset} \left[\omega_h(x)\Indicator[y(x) = +1]\right]
\end{align*}
This implies
\begin{align} \label{eq:induced-distribution-labels-equivalence}
    \Pr_{x\sim \Dataset^*} [y(x) = +1] \geq \Pr_{x\sim \Dataset} [y(x) = +1]
    \Longleftrightarrow
    \Expectation_{\Dataset} \left[(\omega_h(x)-1)\Indicator[y(x) = +1]\right]\geq 0
\end{align}

By similar reasoning, we have
\begin{align*}
    \Pr_{x\sim \Dataset^*} [h(x) = +1] 
    = \Expectation_{\Dataset^*}[\Indicator[h(x)= +1]]
    = \Expectation_{\Dataset} \left[\omega_h(x)\Indicator[h(x) = +1]\right]
\end{align*}
which implies
\begin{align} \label{eq:induced-distribution-predictions-equivalence}
    \Pr_{x\sim \Dataset^*} [h(x) = +1] \geq \Pr_{x\sim \Dataset} [h(x) = +1]
    \Longleftrightarrow
    \Expectation_{\Dataset} \left[(\omega_h(x)-1)\Indicator[h(x) = +1]\right]\geq 0
\end{align}

It is easy to verify that $\Expectation_{x\sim \Dataset} [\omega_h(x)] = 1$, and this gives us
\begin{align} \label{eq:expectation-covariance-equivalences}
    \Expectation_{\Dataset} \left[(\omega_h(x)-1)\Indicator[y(x) = +1]\right] = \text{Cov}_{\Dataset}(\omega_h(x), \Indicator[y(x) = +1]) \\
    \Expectation_{\Dataset} \left[(\omega_h(x)-1)\Indicator[h(x) = +1]\right] = \text{Cov}_{\Dataset}(\omega_h(x), \Indicator[h(x) = +1])
\end{align}

By \eqref{eq:induced-distribution-labels-equivalence}, \eqref{eq:induced-distribution-predictions-equivalence}, and \eqref{eq:expectation-covariance-equivalences}, the condition
\begin{align*}
    \Pr_{x\sim \Dataset^*} [h(x) = +1] \geq \Pr_{x\sim \Dataset} [h(x) = +1]
    \Longleftrightarrow
    \Pr_{x\sim \Dataset^*} [y(x) = +1] \geq \Pr_{x\sim \Dataset} [y(x) = +1]
\end{align*}
is equivalent to the condition
\begin{align*}
\text{Cov}_{\Dataset}(\omega_h(x), \Indicator[y(x) = +1]) \geq 0
\Longleftrightarrow 
\text{Cov}_{\Dataset}(\omega_h(x), \Indicator[h(x) = +1]) \geq 0
\end{align*}
\end{proof}

\subsection{Derivations for the model designer's objective function}
\label{sec:model-designer-objective-derivations}
Now that we have obtained a closed-form expression for both the unconstrained and improving best response from the decision subjects, we can analyze the objective function for the model designer, and the model that would be deployed at equilibrium. Recall that the objective function for the model designer is
\begin{align*}
\min_{w \in \RealNumber^{d+1}} & \quad \Expectation_{x\sim \Dataset}\left[\Indicator(h(\BM(x)) \neq y)\right] + \lambda \Expectation_{x\sim \Dataset}\left[\Indicator(h(\BI(x)) \neq +1)\right]\nonumber \\ 
\end{align*}

By Theorem \ref{thm:best-response-close-form}, $h(\BM(x))$ has the closed form
\begin{align*}
    h(\BM(x))
    &= \begin{cases}
    +1 & \text{if}\ \  w\cdot x\geq -2\sqrt{\WM^\T \SM \WM}\\
    -1 & \text{otherwise}
    \end{cases} \\
    &= 2\cdot \Indicator\left[w\cdot x \geq -2\sqrt{\WM^\T\SM\WM}\right] - 1
\end{align*}
and similarly,
\begin{align*}
   h(\BI(x)) = 2\cdot \Indicator\left[w \cdot x \geq -2 \sqrt{\WI^\T\SI \WI} \right] - 1
\end{align*}
The model designer's objective can then be re-written as follows:
\begin{align*}
    & \mathbb{E}_{x\sim D} \left[\Indicator[{h(\BM(x)) \neq y}] +\lambda  \mathbbm{1}[{h(\BI(x)) \neq +1}]\right]\\
    =& \mathbb{E}_{x\sim \Dataset} \left[1-\frac{1}{2}(1+h(\BM(x))\cdot y) +  \lambda (1-\frac{1}{2}(1 + h(\BI(x))\cdot 1))\right]\\
    =& \mathbb{E}_{x\sim \Dataset} \left[\frac{1}{2} (1+\lambda) - \frac{1}{2}h(\BM(x))\cdot y - \frac{\lambda}{2} h(\BI(x))\right] 
\end{align*}
Removing the constants, the objective function becomes:
\begin{align*}
    & \min_{w} \mathbb{E}_{x\sim \Dataset} \left[\lambda - h(\BM(x))\cdot y - \lambda h(\BI(x))\right] \\
    =& \min_{w} \mathbb{E}_{x\sim \Dataset} \Bigg[
        -\left(2\cdot \Indicator\left[w\cdot x \geq -2\sqrt{\WM^{\T} \SM \WM}\right] - 1 \right)\cdot y(x)
        - 2\lambda \cdot \Indicator\left[w\cdot x \geq -2\sqrt{\WI^\T S_I \WI}\right]
    \Bigg]
\end{align*}

\subsection{Directionally Actionable Features} 
\label{sec:partially-actionable-feature}

In practice, individuals can often only change some features in either a positive or negative direction, but not both.
However, modeling this restriction on the decision subject's side precludes a closed-form solution. Instead, we strongly disincentivize such moves in the model designer's objective function. The idea is that if the model designer is punished for encouraging an illegal action, the announced classifier will not incentivize such moves from decision subjects. The result is that decision subjects encounter an \emph{implicit} directional constraint on the relevant variables. To that end, we construct a vector $\Direction \in \{-1, 0, +1\}^d$ where $\Direction_i$ represents the prohibited direction of change for the corresponding feature $x_i$; that is, $\Direction_i = +1$ if $x_i$ should not be allowed to increase, $-1$ if it should not decrease, and $0$ if there are no directional constraints. We then append the following penalty term to the model designer's objective in \cref{eq:objective}:
\begin{align}
    -\eta \cdot \sum_{i=1}^d \max(\Direction_i \cdot (\Delta(x) - x)_i, 0)\label{eq:direction_penalty}
\end{align}
where $\eta>0$ is a hyperparameter representing the weight given to this penalty term. \cref{eq:direction_penalty} penalizes the weights of partially actionable features so that decision subjects would prefer to move towards a certain direction. We provide more evaluation details in \cref{tab:direction_flipset}.

\section{Additional Related Work}
\label{sec:additional-related-work}
\paragraph{Strategic Classification} There has been extensive research on strategic behavior in classification \cite{hardt2016strategic,cai2015optimum,chen2018strategyproof,dong2018strategic,Dekel2010Incentive,chen2020learning}. \cite{hardt2016strategic} was the first to formalize strategic behavior in classification based on a sequential two-player game (i.e. a Stackelberg game) between decision subjects and classifiers. Since then, other similar Stackelberg formulations have been studied \cite{balcan2015commitment}. \cite{dong2018strategic} considers the setting in which decision subjects arrive in an online fashion and the learner lacks full knowledge of decision subjects' utility functions. More recently, \cite{chen2020learning} proposes a learning algorithm with non-smooth utility and loss functions that adaptively partitions the learner’s action space according to the decision subject’s best responses.  

\paragraph{Recourse} %Recourse represents the ability of an individual to change the outcome of a classifier by altering actionable features. 
The concept of \emph{recourse} in machine learning was first introduced in \cite{ustun2019actionable}. There, an integer programming solution was developed to offer actionable recourse from a linear classifier. Our work builds on theirs by considering strategic actions from decision subjects, as well as by aiming to %dividing actionable features into improvable features and manipulated features and developing solutions to 
incentivize honest improvement. %After that, a number of works have been devoted to recourse analysis in machine learning. 
\cite{venkatasubramanian2020philosophical} discusses a more adequate conceptualization and operationalization of recourse. \cite{karimi2020survey} provides a thorough survey of algorithmic recourse in terms of its definitions, formulations, solutions, and prospects. 
Inspired by the concept of recourse, \cite{dean2020recommendations} develops a reachability problem to capture the ability of models to accommodate arbitrary changes in the interests of individuals in recommender systems. %\cite{rudin2019stop} argues the sufficiency of recourse in explainable machine learning. 
\cite{bellamy2018ai} builds toolkits for actionable recourse analysis. Furthermore, \cite{gupta2019equalizing} studies how to mitigate disparities in recourse across populations.

\paragraph{Causal Modeling of Features} A flurry of recent papers have demonstrated the importance of understanding causal factors for achieving fairness in machine learning \cite{wang2019repairing,bhatt2020explainable,bechavod2020causal,miller2020strategic,shavit2020causal}. \cite{miller2020strategic} studies distinctions between gaming and improvement from a causal perspective. \cite{shavit2020causal} provides efficient algorithms for simultaneously minimizing predictive risk and incentivizing decision subjects to improve their outcomes in a linear setting. 
In addition, \cite{karimi2020algorithmic} develops methods for discovering recourse-achieving actions with high probability given limited causal knowledge. In contrast to these works, we explicitly separate improvable features from manipulated features when maximizing decision subjects' payoffs.

\paragraph{Incentive Design} Like our work, \cite{kleinberg2020classifiers} discusses how to incentivize decision subjects to improve a certain subset of features. Next, \cite{haghtalab2020maximizing} shows that an appropriate projection is an optimal linear mechanism for strategic classification, as well as an \emph{approximate} linear threshold mechanism. Our work complements theirs by providing appropriate linear classifiers that balance accuracy and improvement. \cite{liu2020disparate} considers the equilibria of a dynamic decision-making process in which individuals from different demographic groups invest rationally, and compares the impact of two interventions: decoupling the decision rule by group and subsidizing the cost of investment.

\paragraph{Algorithmic Fairness in Machine Learning} %We analyze the recourse disparity among protected groups through the lens of group fairness.
Our work contributes to the broad study of algorithmic fairness in machine learning.
Most common notions of group fairness include disparate impact~\cite{feldman2015certifying}, demographic parity~\cite{agarwal2018reductions}, disparate mistreatment~\cite{zafar2019fairness}, equality of opportunity~\cite{hardt2016equality} and calibration~\cite{chouldechova2017fair}. Among them, disparities in the recourse fraction can be viewed as equality of false positive rate (FPR) in the strategic classification setting. Disparities in costs and flipsets are also relevant to counterfactual fairness~\cite{kusner2017counterfactual} and individual fairness~\cite{dwork2012fairness}. Similar to our work,  \cite{vonkugelgen2020fairness} also consider the intervention cost of recourse in flipping the prediction across subgroups, investigating the fairness of recourse from a causal perspective.

\section{Additional Experimental Results}
\label{sec:additional-experimental-result}
In this section, we provide additional experimental results. 
\subsection{Basic information of each dataset}
\begin{table*}[!htb]
    \centering
    \caption{Basic information of each dataset.}
    \label{tab:dataset-info}
    \resizebox{!}{!}{
    \begin{tabular}{l c c p{0.5\linewidth}}
    \toprule
        Dataset & Size & Dimension & Prediction Task \\
    \midrule
        \multirow{1}{*}{\textds{credit}} & $20,000$ & $16$ & To predict if a person can repay their credit card loan.\\
    \midrule
        \multirow{1}{*}{\textds{adult}} & $48,842$ & $14$ & To predict whether income exceeds $50K/yr$ based on census data.\\
    \midrule
        \multirow{1}{*}{\textds{german}} & $1,000$ & $26$ & To predict whether a person is good or bad credit risk. \\
    \midrule
        \multirow{1}{*}{\textds{spam}} & $4601$ & $57$ & To predict if an email is a spam or not.\\
    \bottomrule
    \end{tabular}
    }
\end{table*}

\subsection{Computing Infrastructure}
We conducted all experiments on a 3 GHz 6-Core Intel Core i5 CPU. All our methods have relatively modest computational cost and can be trained within a few minutes.

\subsection{Flipsets}
We also construct flipsets for individuals in the \textds{german} dataset using the closed-form solution \cref{eq:best_response_delta_x} under our trained classifier. The individual characterized as a ``bad consumer'' ($-1$) is supposed to decrease their missed payments in order to flip their outcome of the classifier with respect to a non-diagonal cost matrix. In contrast, even though the individual improves their loan rate or liable individuals, the baseline classifier will still reject them. %In general, we find that there is no significant discrepancy between an individual's flipsets with diagonal or non-diagonal cost matrices. 
We also provide flipsets for partially actionable features on the \textds{credit} dataset in \cref{tab:direction_flipset}. The individual will undesirably reduce their education level when the classifier is unaware of the partially actionable features. In contrast, the individual decreases their total overdue months instead when the direction penalty is imposed during training.

\begin{table}[htb]
    \centering
    \captionof{table}{Flipset for a person denied credit by ManipulatedProof on the \textds{german} dataset. The red up arrows \textcolor{red}{$\uparrow$} represent increasing the values of features, while the red down arrows \textcolor{red}{$\downarrow$} represent decreasing.}
    \label{tab:german_flipset}
    \begin{tabular}{l c l l l}
    \toprule
        Feature & Type & Original & LightTouch & ManipulatedProof \\
    \midrule
        \textit{LoanRateAsPercentOfIncome } & I & 3 & 3 & 2~\textcolor{red}{$\downarrow$} \\
        \textit{NumberOfOtherLoansAtBank } & I & 1 & 1 & 1 \\
        \textit{NumberOfLiableIndividuals } & I & 1 & 0~\textcolor{red}{$\downarrow$} & 2~\textcolor{red}{$\uparrow$} \\
        \textit{CheckingAccountBalance} $\geq 0$ & I & 0 & 0 & 0 \\
        \textit{CheckingAccountBalance} $\geq 200$ & I & 0 & 0 & 0 \\
        \textit{SavingsAccountBalance} $\geq 100$ & I & 0 & 0 & 0 \\
        \textit{SavingsAccountBalance} $\geq 500$ & I & 0 & 0 & 0 \\
        \textit{MissedPayments } & I & 1 & 0~\textcolor{red}{$\downarrow$} & 1 \\
        \textit{NoCurrentLoan } & I & 0 & 0 & 0 \\
        \textit{CriticalAccountOrLoansElsewhere } & I & 0 & 0 & 0 \\
        \textit{OtherLoansAtBank } & I & 0 & 0 & 0 \\
        \textit{OtherLoansAtStore } & I & 0 & 0 & 0 \\
        \textit{HasCoapplicant } & I & 0 & 0 & 0 \\
        \textit{HasGuarantor } & I & 0 & 0 & 0 \\
        \textit{Unemployed } & I & 0 & 0 & 0 \\
    \midrule
        \textit{LoanDuration } & M & 48 & 47~\textcolor{red}{$\downarrow$} & 47~\textcolor{red}{$\downarrow$} \\
        \textit{PurposeOfLoan } & M & 0 & 0 & 0 \\
        \textit{LoanAmount } & M & 4308 & 4307~\textcolor{red}{$\downarrow$} & 4307~\textcolor{red}{$\downarrow$} \\
        \textit{HasTelephone } & M & 0 & 0 & 0 \\
    \midrule
        \textit{Gender } & U & 0 & 0 & 0 \\
        \textit{ForeignWorker } & U & 0 & 0 & 0 \\
        \textit{Single } & U & 0 & 0 & 0 \\
        \textit{Age } & U & 24 & 24 & 24 \\
        \textit{YearsAtCurrentHome } & U & 4 & 4 & 4 \\
        \textit{OwnsHouse } & U & 0 & 0 & 0 \\
        \textit{RentsHouse } & U & 1 & 1 & 1 \\
        \textit{YearsAtCurrentJob} $\leq 1$  & U & 1 & 1 & 1 \\
        \textit{YearsAtCurrentJob} $\geq 4$  & U & 0 & 0 & 0 \\
        \textit{JobClassIsSkilled } & U & 1 & 1 & 1 \\
    \midrule
        \textit{GoodConsumer} & - & $-1$ & $+1$~\textcolor{red}{$\uparrow$} & $-1$ \\
    \bottomrule
    \end{tabular}
    \caption{Caption}
    \label{tab:my_label}
\end{table}

\begin{table}[htb]
    \centering
    \captionof{table}{Flipset for an individual on Credit dataset with partially actionable features. The red up arrows \textcolor{red}{$\uparrow$} represent any increasing values, while the red down arrows \textcolor{red}{$\downarrow$} represent any decreasing values.}
    \label{tab:direction_flipset}
    \begin{tabular}{l c c l l l}
        \toprule
            Feature & Type & $\Direction$ & Original & $\eta=0$ & $\eta=100$\\
        \midrule
            \textit{EducationLevel } & I & $+1$ & $3$ & $2$~\textcolor{red}{$\downarrow$} & $3$ \\
            \textit{TotalOverdueCounts } & I & 0 & $1$ & $1$ & $1$ \\
            \textit{TotalMonthsOverdue } & I & 0 & $1$ & $1$ & $0$~\textcolor{red}{$\downarrow$} \\
        \midrule
            \textit{MaxBillAmountOverLast6Months } & M & 0 & $0$ & $0$ & $0$\\
            \textit{MaxPaymentAmountOverLast6Months } & M & 0 & $0$ & $0$ & $0$\\
            \textit{MonthsWithZeroBalanceOverLast6Months } & M & 0 & $0$ & $0$ & $0$\\
            \textit{MonthsWithLowSpendingOverLast6Months } & M & 0 & $6$ & $5$~\textcolor{red}{$\downarrow$} & $6$\\
            \textit{MonthsWithHighSpendingOverLast6Months } & M & 0 & $0$ & $0$ & $0$\\
            \textit{MostRecentBillAmount } & M & 0 & $0$ & $0$ & $0$\\
            \textit{MostRecentPaymentAmount } & M & 0 & $0$ & $0$ & $0$\\
        \midrule
            \textit{Married} & U & 0 & $1$ & $1$ & $1$ \\
            \textit{Single} & U & 0 & $0$ & $0$ & $0$ \\
            \textit{Age} $\leq 25$ & U & 0 & $0$ & $0$ & $0$ \\
            $25 \leq$ \textit{Age} $\leq 40$ & U & 0 & $0$ & $0$ & $0$ \\
            $40 \leq$ \textit{Age} $<60$ & U & 0 & $0$ & $0$ & $0$ \\
            \textit{Age} $\geq 60$ & U & 0 & $1$ & $1$ & $1$ \\
            \textit{HistoryOfOverduePayments } & U & 0 & $1$ & $1$ & $1$ \\
        \midrule
            \textit{NoDefaultNextMonth} & - & - & $-1$ & $+1$~\textcolor{red}{$\uparrow$} & $+1$~\textcolor{red}{$\uparrow$} \\
        \bottomrule
    \end{tabular}
\end{table}

\end{document}